\documentclass{article}
\usepackage{graphicx} %

\title{Optimistic Interior Point Methods for Sequential Testing by Betting}
\author{Can Chen\footnote{University of California San Diego, cac024@ucsd.edu.} \text{  and} Jun-Kun Wang \footnote{
Corresponding author. University of California San Diego, jkw005@ucsd.edu.}}
\date{}

\usepackage[margin=1in]{geometry}
 \usepackage{times}
 \usepackage{amsmath,amsfonts,amsthm,bm,amssymb}

\usepackage{dsfont}
\usepackage{mathtools}
\usepackage[sort]{cite}
\usepackage{xcolor}

\usepackage[colorlinks,urlcolor=blue,citecolor=blue,linkcolor=blue]{hyperref}
 \usepackage{color}

\usepackage{cleveref}

\usepackage{thm-restate}

\usepackage{algorithm}
\usepackage{algorithmic}
\usepackage{wrapfig}
\usepackage{booktabs}
\usepackage{mdframed}
\usepackage{graphicx}
\usepackage{subcaption}

\newtheorem{theorem}{Theorem}
\newtheorem{lemma}{Lemma}

\def\H{\mathcal{H}}

\def\E{\mathbb{E}}
\def\K{\mathcal{K}}
\def\P{\mathcal{P}}
\def\I{\mathcal{I}}
\newcommand{\reals}{\mathcal{R}}

\begin{document}

\maketitle

\begin{abstract}
The technique of ``testing by betting" frames nonparametric sequential hypothesis testing as a multiple-round game, where a player bets on future observations that arrive in a streaming fashion, accumulates wealth that quantifies evidence against the null hypothesis, and rejects the null once the wealth exceeds a specified threshold while controlling the false positive error. Designing an online learning algorithm 
that achieves a small regret in the game can help rapidly accumulate the bettor's wealth, which in turn can shorten the time to reject the null hypothesis under the alternative $H_1$. However, many of the existing works employ the Online Newton Step (ONS) to update within a halved decision space to avoid a gradient explosion issue, which is potentially conservative for rapid wealth accumulation. In this paper, we introduce a novel strategy utilizing interior-point methods in optimization that allows updates across the entire interior of the decision space without the risk of gradient explosion. 
Our approach not only maintains strong statistical guarantees but also facilitates faster null hypothesis rejection, while being as computationally lightweight as ONS thanks to its closed-form updates.
\end{abstract}

\section{Introduction}

Sequential hypothesis testing examines a sequence of observations with the goal of rigorously assessing the validity of the null hypothesis $\mathcal{H}_0$ against the alternative $\mathcal{H}_1$. 
Although a classical problem in statistics, sequential hypothesis testing has gained renewed significance in contemporary contexts \cite{ramdas2023game,grunwald2024anytime}. One of the main reasons perhaps is the recent surge in algorithmic systems and machine learning applications \cite{Chugg2023,CW2025,teneggi2024bet,bar2024protected}, which has given rise to a myriad of nuanced desiderata for the hypothesis testing methods.
These requirements include the ability to continuously monitor incoming data rather than adhere to a fixed-sample size setting; the need to conduct a nonparametric test instead of making distributional assumptions such as assuming that the data follow a certain distribution; and the call to maintain a valid test at any stopping time.
This has hence catalyzed substantial research interest in the algorithmic development of sequential hypothesis testing for simultaneously satisfying these criteria while enjoying provable statistical guarantees \cite{ramdas2024hypothesis}.

In pursuit of these desiderata, sequential hypothesis testing via betting has evolved into a cornerstone methodology and has seen substantial advancements in recent years \cite{shafer2021testing,vovk2021values,shekhar2023nonparametric,ramdas2023game,grunwald2024anytime}. 
In particular, the techniques of ``testing by betting" have found use in many machine learning applications, which include: auditing the fairness of a classifier \cite{Chugg2023}, online detection of whether a text sequence source is an LLM \cite{CW2025}, monitoring distribution shifts for test-time adaptation \cite{bar2024protected}, examining the importance of semantic concepts in a model's prediction for explainable AI \cite{teneggi2024bet}, evaluating voxel responses in neuroimaging data \cite{fischer2024multiple}, testing conditional independence of ride duration and membership for a bikeshare system \cite{grunwald2024anytime}, 
multi-arm bandit problems \cite{cho2024peeking}, adversarial attacks \cite{pandeva2024deep}, 
estimating the mean of a bounded random variable \cite{waudby2024estimating}, and more \cite{shaer2023model,podkopaev2023sequential,podkopaev2024sequential,PXL24,dai2025individual}.

The game-theoretic approaches typically frame the sequential hypothesis testing as a repeated game between the online learner and nature \cite{shafer2019game,shafer2021testing,shekhar2023nonparametric}. Informally, the crux lies in designing a wealth process for an online learner (i.e., the bettor) in the game such that, when the learner's wealth becomes sufficiently large, it increases their confidence in rejecting the null hypothesis. 
However, to be more concrete and get the ball rolling, 
let us use a non-parametric two-sample testing task in many prior works as an example (e.g., \cite{shekhar2023nonparametric,Chugg2023,CW2025,teneggi2024bet}). In this setting, one observes two sequences of samples $\{ x_t\}_{t\geq 1}$ and $\{ y_t\}_{t\geq1}$, where $x_t$ and $y_t$ are pair of bounded random variables observed at time $t$, with their population means denoted as $\mu_x:=\mathbb{E}[x_t]$ and $\mu_y:=\mathbb{E}[y_t]$ respectively. The hypothesis testing task
of \emph{difference-in-means testing}
 can be expressed as: 
\begin{mdframed}
\begin{equation} \label{scenario-1a}
\textbf{(Difference-in-means testing):}
\qquad 
\H_0: \mu_x = \mu_y, \quad   \text{versus} \quad  \H_1: \mu_x \neq \mu_y.
\end{equation}
\end{mdframed}
As the related works \cite{Chugg2023,shekhar2023nonparametric,teneggi2024bet}, we assume $x_t \in [0, 1]$ and $y_t \in [0, 1]$, and we note that the modifications for extending the range is relatively straightforward (see e.g., \cite{CW2025}).

In the game of testing by betting, the online learner usually starts with an initial wealth of $W_1 = 1$ and makes sequential decisions over time. At each round $t$, prior to observing an observation $g_t$, the learner selects a decision point $\theta_t \in \K$, where $\K$ denotes the learner's decision space.  The learner's wealth dynamic
of the difference-in-means testing is governed by the following update rule:
\begin{equation} \label{scenario-1b}
W_{t+1} = W_t \cdot \left(1 - \theta_t  (x_t - y_t) \right).
\end{equation}
One can interpret the magnitude of $\theta_t$ as the amount of the wealth that the online learner bets on the outcome $x_t - y_t$. The high-level idea of testing by betting is that when the wealth $W_t$ surpasses a certain threshold, i.e., when $W_t \geq \frac{1}{\alpha}$, where $\alpha > 0$ is a parameter, the null hypothesis $\mathcal{H}_0$ can be rejected with high confidence. However, to ensure this strategy enjoy strong statistical guarantees, the common algorithmic design principle in the literature is to let the wealth $(W_t)_{t\geq1}$ be a nonnegative supermartingale when $\mathcal{H}_0$ is true.
Then, one can use Ville's inequality~\cite{Ville1939} to control the false positive rate.
More precisely, Ville's inequality states that if $(W_t)_{t\geq 1}$ is a nonnegative supermartingale, then $P(\exists t : W_t \geq {1}/{\alpha}) \leq \alpha \mathbb{E}[W_1]$. 
Hence, with the initial wealth $W_1=1$ and an appropriate choice of the learner's constraint set $\K$ to make sure $W_t$ is nonnegative, the Type-I error at any stopping time can be bounded by $\alpha$. Several previous works simply let the decision space $\K$ be $\K=[-1/2, 1/2]$, e.g., \cite{pandeva2024deep,shekhar2023nonparametric,podkopaev2024sequential,podkopaev2023sequential,waudby2024estimating,teneggi2024bet,Chugg2023,CW2025}.
However, as we shall elucidate and underscore soon, the selection of the decision space in much of the related literature is potentially conservative, thereby leaving considerable room for algorithmic improvements in conjunction with novel update schemes.

When the alternative $\mathcal{H}_1$ is true, on the other hand, one would wish to quickly reject the null hypothesis $\mathcal{H}_0$. This motivates the idea of instantiating the bettor as an algorithm in online learning (a.k.a.~no-regret learning), see e.g., \cite{Hazan2016OnlineConvexOptimization,orabona2019modern}.
In online learning, the learner aims to obtain a sublinear regret (w.r.t.~the number of rounds $T$): $ 
\mathrm{Regret}_T(\theta_*) := \sum_{t=1}^T \ell_t(\theta_t) - \sum_{t=1}^T \ell_t(\theta_*)$,
where $\theta_* \in \mathcal{K}$ is a comparator in the learner's decision space $\K$.
Let $g_t:= x_t - y_t$.
 An observation is that if we define 
$\ell_t(\theta_t):= - \ln (1-g_t \theta_t)$, then there is a nice relation between \emph{regret minimization} and \emph{wealth maximization} in the betting game. In particular, we have $
\ln ( W_T) = \sum_{t=1}^T  \ln (1-g_t \theta_t) = - \sum_{t=1}^T \ell_t(\theta_t).$
Hence, a smaller regret bound translates to a faster growth of the learner's wealth $W_T$, which in turn can lead to a shorter time to reject the null when the alternative holds.
Existing works of testing by betting such as \cite{pandeva2024deep,shekhar2023nonparametric,podkopaev2024sequential,podkopaev2023sequential,waudby2024estimating,teneggi2024bet,Chugg2023,CW2025} adopt Online Newton Steps (ONS)
\cite{hazan2007logarithmic} (replicated in Algorithm~\ref{alg:ONS1})
to update $\theta_t$ in the decision space $\K=[-1/2, 1/2]$.
Such a \emph{heuristic} choice of the decision space $\K$ arises from the fact that if the learner's action $\theta_t$ is allowed to be $\K=[-1,1]$, there is a chance  that the learner's loss will explode, as the value of $1-g_t \theta_t$ could be (close to) $0$. Yet, the online learner has to determine its action $\theta_t$ \emph{before} observing $g_t$. 
Therefore, the conservative decision space is adopted in these works.

The limitation in prior works lies in the fact that when $g_t$ is relatively small, the learner could benefit from allocating a larger magnitude of $\theta_t$ to accelerate the growth of their wealth, which in turn helps reduce the time needed to reject the null hypothesis when the alternative $\H_1$ holds. For example, if the learner adopts a more aggressive betting strategy by selecting $\theta_t = 1$, and the outcome turns out to be $g_t = -0.1$, they would undoubtedly gain a larger fortune compared to being restricted to $\theta_t = \frac{1}{2}$. While this idea is intuitive, the challenge lies in designing an update strategy on the largest possible decision space $\K = [-1, 1]$ such that a larger bet can be made at some points while avoiding the risk of loss explosion. 
In other words, \emph{can we design a more sample-efficient testing-by-betting algorithm than ONS that provides strong statistical guarantees while enjoying a closed-form, lightweight update as ONS?}
Our work seeks to provide a solution to this question and tackle the common issue in prior works of ONS.

We will leverage the techniques of interior-point methods in optimization literature \cite{nesterov1994interior,nemirovski2004interior,wright1997primal} to avoid the loss explosion issue and a vacuous regret, while allowing a large decision space in the betting game. 
In particular, we will use the toolkit of self-concordant barrier functions and propose two novel methods for testing by betting. We note that the toolkit of self-concordant barrier functions have found powerful for getting provable fast convergence of the Newton's method in optimization \cite{nemirovski2008interior}, designing efficient bandit online learning algorithms \cite{abernethy2012interior}, and providing non-trivial non-asymptotic analyses for logistic regression \cite{bach2010self}. However, to our knowledge, this is the first time that this technique has been incorporated in the area of sequential hypothesis testing via betting.

To give the reader a flavor of our contributions, we provide a summary of our results here. First, we highlight a potential limitation that was overlooked in the literature of sequential hypothesis testing by betting
and propose Follow-the-Regularized-Leader (FTRL) \cite{hazan2010extracting,abernethy2012interior} with a barrier function as the regularization that tailored to the betting game. We show that our proposed method is an anytime-valid level-$\alpha$ sequential hypothesis test with asymptotic power of 1. Moreover, we identify a key condition with concrete examples where the proposed method achieves a shorter expected time to reject the null hypothesis $\mathcal{H}_0$ under the alternative $\H_1$, compared to Online Newton Steps (ONS), which has been adopted in many prior works. Furthermore, the proposed method has a \emph{closed-form update} per iteration like ONS; therefore, the implementation cost is low.
Our simulations also confirm the theoretical results. 
Second, we also incorporate the idea of \emph{optimistic} online learning \cite{rakhlin2013online,syrgkanis2015fast,wang2018acceleration} to propose a variant, which we denote as Optimistic-FTRL + Barrier.
When the sequence of samples becomes ``predictive'', the variant aims to exploit it to achieve faster rejection time for $\mathcal{H}_0$, compared to its counterpart without the optimistic learning mechanism.
We further conduct simulations to evaluate the proposed methods against existing methods. The code to reproduce these results is provided in 
this link \url{https://github.com/jimwang123/Interior_Point_Methods_for_Sequential_Hypothesis_Testing.git}.

\section{Preliminaries}

In this section, we provide additional necessary background and notations before introducing our algorithms in the following sections.

\noindent
\subsection{Definition: Level-$\alpha$ sequential hypothesis test with asymptotic power one}
Consider a sequence of observations $\{Z_i: i \geq 1\}$ and let $\mathcal{F} = (\mathcal{F}_t)_{t \geq 0}$ be the forward filtration where each $\mathcal{F}_t = \sigma(Z_1, \ldots, Z_t)$ captures the information available up to time $t$. For any process $W := (W_t)_{t \geq 1}$ adapted to this filtration, we say $W$ is a $\P$-martingale if it satisfies
$\E_{\P}[W_{t} | \mathcal{F}_{t-1}] = W_{t-1}$ for all $t \geq 1,$
and a $\P$-supermartingale if the equality is replaced by an inequality:
$\E_{\P}[W_{t} | \mathcal{F}_{t-1}] \leq W_{t-1}, \forall t \geq 1.$
Let a binary-valued variable $\I_{t} \in \{0,1\}$ be an indicator of the rejection of the null hypothesis $H_0$ at a stopping time $t$. A sequential hypothesis test constructed from a martingale process $W$ maintains its level-$\alpha$ test if
$\sup_{\P \in H_0} P\left(\exists t \geq 1: \I_{t} = 1 \right) \leq \alpha$.
On the other hand, a sequential test achieves asymptotic power $1-\beta$ if 
$\sup_{\P \in H_1} P\left(\forall  t \geq 1: \I_{t} = 0 \right) \leq \beta$.
Of particular interest is the case where $\beta = 0$, which corresponds to asymptotic power one. Specifically, this guarantees that under $H_1$, the underlying sequential hypothesis testing method will eventually reject the null hypothesis $H_0$. As we will demonstrate, the proposed two new algorithms in this work are both level-$\alpha$ tests with asymptotic power one.

\noindent
\subsection{Definition: Self-concordant functions}
A \textit{self-concordant function} $R(\cdot):\text{int}(\mathcal{K}) \rightarrow \mathbb{R}$ is a thrice continuously differentiable convex function such that for all $h \in \mathbb{R}^d$ and $\theta \in \text{int}(\mathcal{K})$, 
\begin{equation}
    \left| D^3 R(\theta)[h, h, h] \right| \leq 2 \left( D^2 R(\theta)[h, h] \right)^{3/2}
    \label{eq:self_concordant},
\end{equation}
where $D^3 R(\theta)[h, h, h]$ is the third-order differential, i.e.,
\begin{equation}
 D^3 R(\theta)[q, r, s] := \frac{\partial^3}{\partial \delta_1 \partial \delta_2 \partial \delta_3} |_{\delta_1=\delta_2=\delta_3=0} R(\theta+ \delta_1 q + \delta_2 r  + \delta_3 s  )
\end{equation}
  is the third-order differential taken at $\theta$ along the directions $q,r,s$; and similarly, $D^2 R(\theta)[h, h]$ is the second-order differential.
Given a self-concordant function $R(\cdot)$, for any point $\theta \in\text{int}(\K)$, we can define a norm $\| \cdot \|_{\theta}$ and its dual norm $\| \cdot \|_{\theta}^*$ as:
\begin{equation}
\| h \|_{\theta} := \sqrt{ h^\top \nabla^2 R(\theta) h } \, \text{ and } \,
\| h \|_{\theta}^* := \sqrt{ h^\top \nabla^{-2} R(\theta) h },
\end{equation}
where $\nabla^{-2} R(\theta) := (\nabla^2 R(\theta) )^{-1}$ is the inverse of the Hessian.
We refer the reader to \cite{nesterov1994interior,nemirovski2008interior,nemirovski2004interior} for the exposition of self-concordant functions. 

\noindent
\subsection{Sequential Hypothesis Testing by Betting}

The notion of testing by betting can be traced back to \cite{cover1974universal} 
and \cite{robbins1974expected}, as wells as \cite{kelly1956new}, who shows that in a repeated game with a binary outcome and certain odds, a gambler can grow their fortune exponentially fast. However, many algorithmic and theoretical foundations have only developed over the recent years, and we refer the reader to a nice monograph by \cite{ramdas2024hypothesis} for the exposition. We note that the wealth in testing by betting  described above is an e-value in the modern statistics literature, see e.g., \cite{vovk2021values,ramdas2024hypothesis,grunwald2024authors,wasserman2020universal,ramdas2020admissible}. 
The idea of betting has also facilitated the design and analysis of algorithms in other domains such as portfolio selection \cite{orabona2023tight} (see also  Subsection~\ref{sub:related_works_on_online_portfolio_selection_} in the following), the construction of confidence sequences \cite{jang2023tighter}, and the design of parameter-free optimization algorithms \cite{orabona2016coin}.
We refer the reader to the references therein for more details.

We now provide more background on the mechanism of sequential hypothesis testing by betting.  
To showcase the flexibility of \emph{testing by betting},
in addition to the example of difference-in-means testing described in the introduction,  
we will also consider a scenario in which one receives a sequence of non-negative samples $
x_t \in [0,1]$.
Formally, the hypothesis testing task is:
\begin{mdframed}
\begin{equation} \label{scenario-2a}
\textbf{(One-sided testing):}
\qquad 
\H_0: \mu_x \leq \mu_0, \quad   \text{versus} \quad  \H_1: \mu_x > \mu_0,
\end{equation}
\end{mdframed}
with the corresponding wealth dynamic being:
\begin{equation} \label{scenario-2b}
W_{t+1} = W_t \cdot \left(1 - \theta_t (\mu_0-x_t) \right),
\end{equation}
where $\theta_t \in [0,1]$ and $\mu_0 \in [0,1]$ is user-specified or given.
We note that in both the \emph{difference-in-means} testing and \emph{one-sided} testing, the wealth dynamic can be written as 
\begin{equation} \label{b}
W_{t+1} = W_t \cdot \left(1 - \theta_t g_t \right),
\end{equation}
where $g_t = x_t - y_t$ in the difference-in-means testing and $g_t = \mu_0 - x_t$ in the one-sided testing. 
The decision space $\K$ in each case will be determined by requiring that $\{W_t\}_{t \geq 1}$ forms a non-negative supermartingale (to be elaborated soon).

The meta-algorithm that forms the basis of several prior works is detailed in Algorithm~\ref{alg:Betting} \cite{pandeva2024deep,shekhar2023nonparametric,podkopaev2024sequential,podkopaev2023sequential,waudby2024estimating,teneggi2024bet,Chugg2023,CW2025}. 
At each round $t$, Algorithm~\ref{alg:Betting} selects a point $\theta_t$ based on an online learning algorithm $\mathrm{OAlg}$.
Subsequently, it observes the samples $x_t$ and $y_t$, and hence it sees the loss function $\ell_t(\cdot)$ as well. The bettor's wealth is then updated according to the dynamics $W_{t+1} = W_t \left(1 - g_t \theta_t \right)$. Then, $\mathrm{OAlg}$ updates the next action $\theta_{t+1}$ by using $\ell_t(\cdot)$ and potentially the history of past ones. 
If the wealth $W_t$ exceeds $1/\alpha$, it declares that the null hypothesis $\H_0$ is false (Line 8). Moreover, if there is a time budget (i.e., $T < \infty$) and the timer runs out, it may reject $H_0$ under a condition (Line 11-13). For $\mathrm{OAlg}$, the aforementioned works 
\cite{pandeva2024deep,shekhar2023nonparametric,podkopaev2024sequential,podkopaev2023sequential,waudby2024estimating,teneggi2024bet,Chugg2023,CW2025} 
adopt ONS (see Algorithm~\ref{alg:ONS1} for the replication). For its comparison to our methods in the next section, we provide the regret bound guarantee of ONS for the betting game in Lemma~\ref{lem:ONS} below, which is a known result in the literature (e.g., proof of Lemma~1 in \cite{CW2025}).

\begin{algorithm}[t]
\caption{Sequential Hypothesis Testing by Betting}
\label{alg:Betting}
\begin{algorithmic}[1] 
\small
\STATE\textbf{Init:} wealth $W_1 \gets 1$, significance level parameter $\alpha \in (0, 1)$, and time budget $T \in [1,\infty]$.
\STATE\textbf{Input:} online learning algorithm $\mathrm{OAlg}$.
\STATE \textbf{Specify:} the decision space $\K$ to 
guarantee that $(w_t)_{t\geq 0}$ is non-negative supermartingale. 
\STATE\textbf{For} $t=1,2,\dots, T$ \textbf{do}
\STATE\quad Play $\theta_t \in \K$ by $\mathrm{OAlg}$.
\STATE\quad Observe a sample $g_t$.
\STATE\quad Update learner's wealth $W_{t+1} = W_t \left(1 - g_t \theta_t \right)$.
\STATE\quad \textbf{If} $W_t \geq 1/\alpha$, \textbf{then} reject the null $\H_0$.
\STATE\quad Send $\ell_t(\cdot) \to \mathrm{OAlg}$, where $\ell_t(\theta_t)= - \ln (1-g_t \theta_t)$.
\STATE\textbf{End For}
\STATE\textbf{If} the null $\H_0$ has not been rejected, \textbf{then}
\STATE \quad Sample $\nu \sim \text{Uniform}[0, 1]$. 
\STATE\quad \textbf{If} $W_T \geq \nu/\alpha$, \textbf{then} reject $\H_0$.
\end{algorithmic}
\end{algorithm}

\begin{mdframed}[backgroundcolor=black!10,rightline=false,leftline=false,topline=false,bottomline=false]
\begin{lemma} \label{lem:ONS} Consider the scenario of bounded random variables $x_t \in [0,1]$ 
and $y_t \in [0,1]$ in the betting game. Online Newton Steps (ONS) (Algorithm~\ref{alg:ONS1}) has
\begin{equation}
   \text{Regret}_T(\theta_*)
 \lesssim \ln\left(\sum_{t=1}^T{g_t^2}\right).
\end{equation}
\end{lemma}
\end{mdframed}
Lemma~\ref{lem:ONS} implies that ONS has an $O(\ln(T))$ regret.

Now we highlight the connection of online learning and sequential hypothesis testing in the following theorem.
Recall that an algorithm is a no-regret learning algorithms if its regret is sublinear with the number of rounds $T$, i.e., $\frac{\mathrm{Regret}_T(\theta_*)}{T} \to 0 \text{ as } T \to \infty$ \cite{orabona2019modern,wang2024no}, and hence a no-regret learner has a vanishing average regret.
Theorem~\ref{thm:1} below shows that Algorithm~\ref{alg:Betting} with $\mathrm{OAlg}$ being any no-regret learning algorithms has strong statistical guarantees,
provided that the significance-level parameter $\alpha$ is not asymptotically close to $0$.
To our knowledge, Theorem~\ref{thm:1} has not been explicitly stated in prior literature. However, its result has indeed been implicitly proven through the analysis of ONS for betting in prior works (e.g., \cite{Chugg2023,shekhar2023nonparametric,dai2025individual,CW2025}) for the difference-in-means setting. We present Theorem~\ref{thm:1} to highlight the modularity of the approach of testing by no-regret learning, and we note that Theorem~\ref{thm:1} covers the prior result when ONS is used to instantiate $\mathrm{OAlg}$. 

In the following, we denote $\omega_*:= \E[ \ln (1 - g \theta_*)  ]$, where $\theta_* := \arg\max_{\theta \in \K } \E\left[ \ln (1 - g \theta) \right]$ is the comparator that maximizes the expected wealth in a single round.

\begin{mdframed}[backgroundcolor=black!10,rightline=false,leftline=false,topline=false,bottomline=false]
\begin{theorem} \label{thm:1}
Assume $(g_t)_{t\geq 1}$ are i.i.d.~random variables
and that $\psi_t:= \ln ( 1 - g_t \theta_*) - \E[  \ln (1-g_t \theta_*) ]$ is a bounded random variable with sub-Gaussian parameter $\sigma>0$.
Algorithm~\ref{alg:Betting} with $\mathrm{OAlg}$ being a no-regret learning algorithm is a level-$\alpha$ sequential test with asymptotic power one
for both the difference-in-mean testing and 
the one-sided testing if the regret of the underlying 
 $\mathrm{OAlg}$ and the parameter $\alpha$ satisfy
$ \frac{\mathrm{Regret}_T(\theta_*)}{T}  
\leq \frac{\omega_*}{2} - \frac{ \ln \left( 1 / \alpha \right)  }{T}$ as $T \to \infty$.
Furthermore, denote $t_*$ the time such that
for all $t \geq t_*$, we have
$\mathrm{Regret}_t(\theta_*) \leq \frac{\omega_* t}{2} - \ln \left(1/\alpha\right)$. 
Then, the expected time to reject $H_0$ is
$$ \E[ \tau ] \leq t_* + \frac{2 \sigma^2}{ \omega_*^2 }.$$
\end{theorem}
\end{mdframed}

Theorem~\ref{thm:1} shows that as long as the significance level parameter $\alpha$ satisfies $\frac{\ln \left( 1 / \alpha \right) }{T} < \frac{\omega_*}{2}$ as $T \to \infty$, the use of any 
no-regret learning algorithm leads to a sequential hypothesis testing method with the strong statistical guarantees. Furthermore, a no-regret learning algorithm with a smaller regret bound yields a smaller bound on the expected time to reject the null hypothesis $H_0$. The proof of the theorem is given below, and some of the ideas in the analysis follow from \cite{CW2025,dai2025individual}, where they only consider either the difference-in-means testing or one-sided testing with ONS as the underlying betting method.

\begin{proof}(of Theorem~\ref{thm:1})

\noindent
\textbf{Level-$\alpha$ Sequential Test.}
We first show that Algorithm~\ref{alg:Betting}
is a valid level-$\alpha$ sequential test.
Let us consider the difference-in-mean testing,
where we recall that $\theta_t \in \K \subseteq [-1,1]$.
Under $H_0$, we have
\begin{align}
    \mathbb{E}[W_t | \mathcal{F}_{t-1}] = \mathbb{E} \left[ (1 - \theta_t g_t)\times W_{t-1} \Bigg| \mathcal{F}_{t-1} \right] =
    \mathbb{E}\left[ W_{t-1} \cdot (1 - \theta_t\cdot (x_t - y_t ) )\Bigg| \mathcal{F}_{t-1} \right] 
    = W_{t-1},
\end{align}
where the last equality is because $\theta_t$ is fully determined 
by the realization of $\mathcal{F}_{t-1}$, and hence, given $\mathcal{F}_{t-1}$,  $\theta_t$ is independent of $x_t - y_t \in [-1,1]$, where the latter 
has conditional expectation zero under $H_0: \mu_x = \mu_y$. We also note that $W_t$ is nonnegative for all $t \geq 1$.

Similarly, for the one-sided testing,
we have $\theta_t \in K \subseteq [0,1]$,
and hence
\begin{align}
    \mathbb{E}[W_t | \mathcal{F}_{t-1}] = \mathbb{E} \left[ (1 - \theta_t g_t)\times W_{t-1} \Bigg| \mathcal{F}_{t-1} \right] =
    \mathbb{E}\left[ W_{t-1} \cdot (1 - \theta_t\cdot (\mu_0 - x_t ) )\Bigg| \mathcal{F}_{t-1} \right] 
    \leq W_{t-1},
\end{align}
where the last equality is because $\theta_t$ is fully determined 
by the realization of $\mathcal{F}_{t-1}$ and is independent of $\mu_0 - x_t$, where the latter 
has conditional expectation a nonnegative number under $H_0: \mu_x \leq \mu_0$. We also note that $W_t$ is nonnegative for all $t \geq 1$,
since $1 - \theta_t (\mu_0 - x_t) \geq 0$.

Therefore, $(W_t)_{t\geq 1}$ is a non-negative supermartingale in both cases.
We can apply Ville's inequality~\cite{Ville1939}
to establish that $P(\exists t \geq 1 : W_t \geq 1/\alpha) \leq \alpha$,
which shows that rejecting $H_0$ once the wealth $W_t$ reaches $1/\alpha$ controls the type-I error at level $\alpha$. Furthermore, if there exists a time budget $T$ and that $H_0$ has not been rejected, we 
can use randomized Ville's inequality \cite{RamdasManole2023} to obtain the guarantee of being a level-$\alpha$ test as well.

\noindent
\textbf{Asymptotic Power One.} 
We now show that the meta algorithm has asymptotic power $1$.

We first connect the regret and the wealth in the testing-by-betting game in the following way: 
\begin{equation} \label{eq:second-repeat-1}
    \ln(W_t) = \ln(W_t(\theta_*)) - \text{Regret}_t(\theta_*) ,
\end{equation}
where $W_t$ is the learner's wealth at round $t$, $\theta_* \in \K$ is the benchmark, and 
$W_t(\theta_*) = \Pi_{s=1}^t (1 - g_s \theta_* )$ is the wealth of the benchmark at the end of round $t$. 
Recall that $\omega_*:= \E[ \ln (1- g \theta_*) ]$ is the expected one-round wealth of the benchmark. 
Taking the expectation on both sides of \eqref{eq:second-repeat-1},
we have
\begin{align*}
\E[ \ln ( W_t  ) ] & = \E[  \ln (W_t(\theta_*)) ]
- \E\left[ \text{Regret}_t(\theta_*)   \right]
\\ & = \E\left[  \sum_{s=1}^t \ln (1 - g_s \theta_*)  \right] - \E[ \text{Regret}_t(\theta_*)   ]
\\ & = t \omega_*  - \E\left[ \text{Regret}_t(\theta_*)   \right],
\end{align*}
where the last equality holds by the assumption that the random variables $(g_s)_{s\geq1 }$ are i.i.d. 

We now analyze the probability that $H_0$ has not been rejected by time $t$, i.e., when the event
$\{ W_t< \frac{1}{\alpha}  \}$ holds for any $s \leq t$. 
We have
\begin{align}
 \mathbb{P}\left[ W_t < \frac{1}{\alpha} \right] 
 & = \mathbb{P}\left[  \ln (W_t) < \ln \left( \frac{1}{\alpha} \right) \right] \notag
\\ & = \mathbb{P}\left[  \ln (W_t) - \E[ \ln (W_t(\theta_*)) ]   
 < \ln\left( \frac{1}{\alpha} \right) - \E[ \ln (W_t(\theta_*)) ] 
 \right] \notag
\\ & = \mathbb{P}\left[  \ln(W_t(\theta_*)) - \text{Regret}_t(\theta_*)  - \E[ \ln (W_t(\theta_*)) ]  
 < \ln\left( \frac{1}{\alpha} \right) - \E[ \ln (W_t(\theta_*)) ] 
 \right] \notag
\\ & = \mathbb{P}\left[  \ln(W_t(\theta_*))  - \E[ \ln (W_t(\theta_*)) ]  
 <
 \text{Regret}_t(\theta_*) +
 \ln\left( \frac{1}{\alpha} \right) - t \omega_* 
 \right].  \label{intermediate-a-1}
\end{align}

We now are ready to show that Algorithm~\ref{alg:Betting} has asymptotic power $1$.
Let $t_*$ be the time such that $\forall t \geq t_*$
$\text{Regret}_t(\theta_*)  \leq \frac{t \omega_*}{2} -  \ln\left( \frac{1}{\alpha} \right).$
Then, from \eqref{intermediate-a-1}, we will have for all $t \geq t_*$:
\begin{align}
\mathbb{P}\left[ W_t < \frac{1}{\alpha} \right]
\leq  
\mathbb{P}\left[ 
 \ln(W_t(\theta_*))  - \E[ \ln (W_t(\theta_*)) ]   
  < - \frac{t}{2} \omega_* 
 \right]. \label{2-12}
\end{align} 

Now let's denote $\psi_t:= \ln ( 1 - g_t \theta_*) - \E[  \ln (1-g_t \theta_*) ]$.
We note that given the assumption that $\psi_t$ is a bounded random variable, $\psi_t$ is a sub-Gaussian with parameter $\sigma$.
Moreover, we have 
$\ln(W_t(\theta_*))  - \E[ \ln (W_t(\theta_*)) ]= \sum_{s=1}^t \psi_s$,
which is a sum of zero-mean i.i.d bounded random variables. 
By Hoeffding's inequality, we have
\begin{equation}
\mathbb{P}\left[  
\frac{1}{t} \sum_{s=1}^t \psi_s \leq - c \right]\leq \exp\left( - \frac{ t c^2 }{ 2\sigma^2 }  \right)     , \label{hoeff2}
\end{equation}
for any constant $c>0$.
Then,
\begin{align}
 \mathbb{P}\left[  \ln(W_t(\theta_*))  - \E[ \ln (W_t(\theta_*)) ]  
 <  - \frac{t}{2} \omega_* 
 \right]
 = 
 \mathbb{P}\left[  \frac{1}{t} \sum_{s=1}^t \psi_s 
 <  - \frac{1}{2} \omega_* 
 \right]
\leq \exp \left(  - \frac{\omega_*^2}{ 2 \sigma^2 } t      \right), \label{2-14}
\end{align}
where we let $c \gets \frac{\omega_*}{2}$ 
in \eqref{hoeff2}.
Combining \eqref{2-12} and \eqref{2-14} leads to 
 \begin{align} \label{ea1a}
 \mathbb{P}\left[ W_t < \frac{1}{\alpha} \right] 
\leq 
\exp \left(  - \frac{\omega_*^2}{ 2 \sigma^2 } t      \right), \forall t \geq t_*.
 \end{align}
Let $H_t:= \{W_t \geq \frac{1}{\alpha} \}$.
Then, 
\begin{align}
\mathbb{P}[ t = \infty ] 
& =
\mathbb{P}\left[ \lim_{t \to \infty} \cap_{s \leq t} \neg H_s \right] \notag
\\ & =  \lim_{t \to \infty}
\mathbb{P}\left[ \cap_{s \leq t} \neg H_s \right] \notag
\\ & \leq 
 \lim_{t \to \infty}
\mathbb{P}\left[  \neg H_t \right] \notag
\\ & = 
 \lim_{t \to \infty}
\mathbb{P}\left[  
W_t < \frac{1}{\alpha} 
\right] \notag
\\ & \leq
\lim_{t \to \infty}
 \exp \left(  - \frac{\omega_*^2}{ 2 \sigma^2 } t      \right) \notag
\\ & = 0,
\end{align}
where the last equality is by \eqref{ea1a}, which shows that the asymptotic power is one.

\noindent
\textbf{Expected Rejection Time $\E[\tau]$.} 

Now we analyze the expected stopping time.
We have
\begin{align}
\E[\tau] & = \sum_{t=1}^{\infty} \mathbb{P}\left[ \tau > t \right] \notag
\\ & = \sum_{t=1}^{\infty}  \mathbb{P}\left[  \cap_{s \leq t} \neg H_s \right] \notag
\\ & \leq \sum_{t=1}^{\infty}  \mathbb{P}\left[  \neg H_t \right] \notag
\\ & =  \sum_{t=1}^{\infty}  \mathbb{P}\left[ 
W_t < \frac{1}{\alpha} \right] \notag
\\ & \leq t_* + \sum_{t=1}^{\infty}
 \exp \left(  - \frac{\omega_*^2}{ 2 \sigma^2 } t      \right) \notag
\\ & \leq t_* + \frac{1}{ \exp \left( \frac{ \omega_*^2}{ 2 \sigma^2 } t  \right) - 1  } \notag
\\ & \leq t_* + \frac{2 \sigma^2}{ \omega_*^2 } , \label{t-nonasymp2}
\end{align}
where we used $\exp(z) \geq 1 + z$ for any $z \geq 0$ for the last inequality.

\end{proof}

We note that the machinery can be naturally extended to the case where the distributions generating $x_t$ and $y_t$ are ever-changing without changing the algorithm with one modification of the testing task:
\begin{equation} \label{setting:distributionshifts}
\H_0: \mu_x(t) = \mu_y(t), \forall t \, \, \text{v.s.} \, \, \H_1: \exists t \geq 1: \mu_x(t) \neq \mu_y(t),
\end{equation}
where $\mu_x(t) := \E[x| \mathcal{F}_{t-1}]$ (and $\mu_y(t)$ similarly defined).
We refer the reader to Section 3.4 in \cite{Chugg2023} for the nice treatment of handling distribution shifts.

\subsection{Related works on Online Portfolio Selection} \label{sub:related_works_on_online_portfolio_selection_}

Online Portfolio Selection (OPS) concerns a sequential investment with the goal oftentimes being competing with
the best constantly re-balanced portfolio,
see e.g., Chapter 13 of \cite{orabona2019modern} for a nice exposition.
The wealth dynamics in OPS are typically defined as: 
\begin{equation} \label{a}
    W_{t+1} = W_t \cdot \langle \lambda_t, r_t \rangle,
\end{equation}
with $\lambda_t \in \Delta$ (the simplex of asset weights) and the vector of price relatives $r_t \in \mathbb{R}^d_+$, where $\Delta:=\{\lambda: \lambda\succ0, \lambda^T1=1\}$ \cite{cover1991universal,agarwal2006algorithms}. 
Early works 
has established that the regret bound of universal portfolios is $\mathcal{O}(d\ln T)$, where $d$ is the number of assets, and $T$ is the number of rounds \cite{cover1991universal}. Since then, there have been persistent interest in OPS on various aspects over the past couple of decades \cite{
cover1991universal, cover1996universal,
agarwal2006algorithms,hazan2007logarithmic, hazan2015online,ito2018regret,luo2018efficient,
van2020open,zimmert2022pushing,mhammedi2022damped, jezequel2022efficient,tsai2023data,orabona2023tight}. Recently, Orabona and Jun~\cite{orabona2023tight} introduced time-uniform concentration inequalities that leverage the regret guarantees of online portfolio algorithms for constructing valid confidence sequences,
which showcases the connection between OPS and testing by betting. Waudby-Smith et al.~\cite{waudby2025universal} further 
provide matching lower and upper bounds on the expected rejection time in sequential hypothesis testing by betting for certain regimes.

\section{Main results: Optimistic Interior Point Methods for Testing by Betting} 

We have shown that the meta-algorithm (Algorithm~\ref{alg:Betting}) equipped with a no-regret learning  algorithm can be a valid level-$\alpha$ test with asymptotic power one, together with a non-asymptotic bound on the rejection time.   
However, from a practical standpoint, one would like to quickly reject the null $\H_0$ when it is false.
This together with the issue of ONS mentioned in the introduction motivates the development of our two algorithms that we are ready to present.

As outlined in the introduction, our algorithms employ 
the technique of barrier functions to constrain the learner's updates to the interior of the domain $\K$. 
For the difference-in-means testing, a natural barrier function for this purpose is
\begin{equation} \label{barrier}
\textbf{For Difference-in-Means Testing (c.f.~\eqref{scenario-1a}) }: \quad 
R(\theta) = -\ln(1-\theta) - \ln (1+\theta),
\end{equation}
where we note that the domain of $R(\cdot)$ is $(-1,1)$.
However, one might wonder: how can this technique overcome the potential loss explosion issue when the learner plays $\theta_t$ such that $|\theta_t| \approx 1$ and still achieve a non-vacuous regret bound?
Recall that we have $\ell_t(\theta_t)= -\ln (1- g_t \theta_t)$, where $g_t:= x_t - y_t \in [-1,1]$ in the difference-in-means testing. The answer lies in exploiting the properties of self-concordant functions for designing no-regret learning strategies, which we detail next.

\subsection{FTRL+Barrier for Testing by Betting} \label{section:main1}

Our first sequential testing-by-betting algorithm is based on Follow-the-Regularized Leader (FTRL), a classical no-regret learning strategy in online learning \cite{shalev2006online,hazan2010extracting,abernethy2012interior,mcmahan2017survey}. The update is depicted in Algorithm~\ref{alg:FTRL}. 

We emphasize that the update $\theta_{t+1}$ of FTRL+Barrier (Algorithm~\ref{alg:FTRL}) has a closed-form solution in the test-by-betting game,
thanks to the use of the barrier function.
Therefore, FTRL+Barrier can be efficiently implemented as ONS, which also has a closed-form update.

\begin{mdframed}[backgroundcolor=black!10,rightline=false,leftline=false,topline=false,bottomline=false]
\begin{lemma} \label{lem:close-form}
Denote $G_t:= \sum_{s=1}^t \nabla \ell_s(\theta_s)$
the cumulative sum of the gradients.
In the difference-in-means testing, FTRL+Barrier
(Algorithm~\ref{alg:FTRL} with the barrier function $R(\cdot)$ defined in \eqref{barrier})
has a close-form update update:
$$
\theta_{t+1} = \begin{cases} & \frac{1 - \sqrt{1 + \left(\eta G_t\right)^2}}{\eta G_t} , 
\quad  \text{ if  } G_t \neq 0,
\\ & 0, \qquad \qquad \qquad  \text{otherwise}. 
\end{cases}
$$
\end{lemma}
\end{mdframed}

\begin{proof}
In the following, we denote $a:=\eta\sum_{s=1}^t \nabla \ell_s(\theta_s)$. Then,
the Lagrangian of the objective function for FTRL+Barrier (line 6 in Algorithm~\ref{alg:FTRL})
over the constraint set $\K= [-1,1]$ is $\mathcal{L}(\theta, \lambda_1, \lambda_2) = a\theta - \ln(1-\theta) - \ln(1+\theta) + \lambda_1(\theta-1) + \lambda_2(-\theta-1)$,
where $\lambda_1, \lambda_2 \geq 0$ 
are Lagrange multipliers.

The complementary slackness of the KKT optimality condition (see e.g., Section 5.5 in \cite{boyd2004convex}) is
 \begin{equation} \label{slackness}
\lambda_1(\theta - 1) = 0 \qquad \text{ and  } \qquad \lambda_2(-\theta - 1) = 0,
\end{equation}
while the stationarity of the KKT condition is 
\begin{equation}
\frac{\partial \mathcal{L}}{\partial \theta} = a + \frac{1}{1 - \theta} - \frac{1}{1 + \theta} + \lambda_1 - \lambda_2 = 0. \label{stationarity}
\end{equation}
Observe that the stationarity condition \eqref{stationarity} cannot be satisfied at $\theta = \pm 1$,
hence the solution $\theta \neq \pm 1$, which in turn implies that $\lambda_1=\lambda_2 = 0$ by 
the complementary slackness \eqref{slackness}.
Then, solving \eqref{stationarity} equivalently becomes solving a quadratic equation, i.e.,
$a\theta^2 - 2\theta - a = 0$,
which has two roots: $\theta = \frac{1 \pm \sqrt{1 + a^2}}{a}$ when $a \neq 0$. However, by the primal feasibility condition $\theta\in[-1,1]$, only $\theta = \frac{1 - \sqrt{1 + a^2}}{a}$ is the valid solution. 
When $a=0$, one can easily conclude that $\theta=0$ from \eqref{stationarity}.
We hence have completed the proof.
\end{proof}

\begin{algorithm}[t]
\caption{FTRL+Barrier as $\mathrm{OAlg}$ for Testing by Betting}
\label{alg:FTRL}
\begin{algorithmic}[1] 
\small
\STATE \textbf{Init:} $\theta_1 \in \K$.
\STATE \textbf{Require:} the parameter $\eta$. 
\STATE \textbf{Specify:} the decision space $\K$ and the barrier function $R(\theta)$.
\STATE\textbf{For} $t=1,2,\dots, T$ \textbf{do}
\STATE\quad Play $\theta_t \in \K$.
\STATE\quad Receive $\ell_t(\theta)= - \ln (1-g_t \theta)$ and incur loss $\ell_t(\theta_t)$.
\STATE \quad Update $\theta_{t+1} \gets \arg\min_{\theta \in \K} \eta
\langle \sum_{s=1}^t \nabla \ell_s(\theta_s), \theta \rangle + R(\theta).
  $
\STATE\textbf{End For}
\end{algorithmic}
\end{algorithm}
For one-sided testing, we propose the following barrier function:
\begin{equation} \label{barrier2}
\textbf{For One-Sided Testing (c.f.~\eqref{scenario-2a}):} \quad 
R(\theta) = -\ln (\theta) - \ln(1-\theta),
\end{equation}
which effectively constrains FTRL+Barrier to update in the interior of the decision space $\K = [0,1]$.
We note that in the one-sided testing, committing $\theta_t = 0$ would not cause the gradient explosion issue (c.f., the game dynamic in \eqref{scenario-2b}), while $\theta_t = 1$ potentially does. The term $\ln(\theta)$ is incorporated so that the update of FTRL+Barrier has a closed-form expression.

\begin{mdframed}[backgroundcolor=black!10,rightline=false,leftline=false,topline=false,bottomline=false]
\begin{lemma} \label{lem:close-form2}
Denote $G_t:= \sum_{s=1}^t \nabla \ell_s(\theta_s)$
the cumulative sum of the gradients.
In the one-sided testing, FTRL+Barrier
(Algorithm~\ref{alg:FTRL} with the barrier function $R(\cdot)$ defiend in \eqref{barrier2})
has a close-form update update:
$$
\theta_{t+1} = 
\begin{cases} & \frac{2 + \eta G_t - \sqrt{ 4 + (\eta G_t)^2 }}{2 \eta G_t}, 
\quad  \text{ if  }
G_t \neq 0,
\\ & \frac{1}{2}, \qquad \qquad \qquad \qquad   \text{otherwise}. 
\end{cases}
$$
\end{lemma}
\end{mdframed}

\begin{proof}
The proof can be derived in a similar fashion as Lemma~\ref{lem:close-form} and is available in Section~\ref{app:lem:close-form2}.
\end{proof}

We now switch to providing the theoretical guarantees of Algorithm~\ref{alg:Betting} with the proposed FTRL+Barrier (Algorithm~\ref{alg:FTRL}). To this end, we require a sequence of technical lemmas.

\begin{mdframed}[backgroundcolor=black!10,rightline=false,leftline=false,topline=false,bottomline=false]
\begin{lemma} \label{lem:barrier}
The barrier function $R(\theta) = -\ln (1-\theta) - \ln (1+\theta)$ in \eqref{barrier} is a self-concordant function for $\K = [-1,1]$, while $R(\theta) = - \ln (\theta) - \ln (1+\theta)$ in \eqref{barrier2} is a self-concordant function for $\K = [0,1]$.
Therefore, the objective function $F_t(\theta):=\left \langle \sum_{s=1}^t \nabla \ell_s(\theta_s), \theta \right  \rangle + R(\theta)$ in FTRL+Barrier is also a self-concordant function
over $\K$ in both cases.
\end{lemma}
\end{mdframed}

\begin{proof}
The proof is deferred to Appendix~\ref{app:lem:barrier}.
\end{proof}

Our result will build upon an established regret bound for FTRL with \emph{any} self-concordant barrier as the regularizer in the literature.

\begin{mdframed}[backgroundcolor=black!10,rightline=false,leftline=false,topline=false,bottomline=false]
\begin{lemma}[Theorem 4.1 in \cite{abernethy2012interior}]\label{lem:FTRL}
Suppose the learner faces a sequence of convex loss functions $\ell_t(\cdot)$ and that the regularization function of FTRL $R(\cdot)$ is a self-concordant barrier. Set the parameter $\eta$ so that $\eta \| \nabla \ell_t(\theta_t) \|_{\theta_t}^* \leq \frac{1}{4} $. Then, FTRL has $$\mathrm{Regret}_T(\theta_*)
\leq
2 \eta \sum_{t=1}^T \| \nabla \ell_t(\theta_t) \|_{\theta_t}^{*2}
+ \frac{  R(\theta_*) - \min_{\theta \in \K} R(\theta)  }{\eta},$$
where $\theta_* \in \mathcal{K}$ is any comparator.
\end{lemma}
\end{mdframed}

To provide intuition on why this result can be effective in tackling the loss (and gradient) explosion issue in our betting game, let us explicitly write out the term $\| \nabla \ell_t(\theta_t) \|_{\theta_t}^{*2}$ in the regret bound of FTRL for the game. Consider the difference-in-means testing. We have:
\begin{align} \label{eqnorm}
\textstyle
\| \nabla \ell_t(\theta_t) \|_{\theta_t}^{*2}
= \underbrace{ \left(\frac{g_t}{1-g_t \theta_t}\right)^2 }_{ = (\nabla \ell_t(\theta_t))^2 }
\underbrace{ \frac{(1-\theta_t)^2(1+\theta_t)^2}{2+2\theta_t^2} }_{ =  ( \nabla^2 R(\theta_t) )^{-1} }.
\end{align}
From \eqref{eqnorm}, one can observe a delicate balance when $\theta_t$ approaches the boundary; the gradient becomes large while the inverse of Hessian  $( \nabla^2 R(\theta_t) )^{-1}$ diminishes. This interplay makes the local gradient norm $\| \nabla \ell_t(\theta_t) \|_{\theta_t}^{*2}$ remain small.

Before proceeding, let us verify that our proposed choice of the barrier functions satisfy $\| \nabla \ell_t(\theta_t) \|_{\theta_t}^* \leq 1 $,
and hence the condition of Lemma~\ref{lem:FTRL} can be satisfied.

\begin{mdframed}[backgroundcolor=black!10,rightline=false,leftline=false,topline=false,bottomline=false]
\begin{lemma} \label{lem:gradient_norm}
Set $\eta \leq \frac{1}{4}$.
Then, $\eta \| \nabla \ell_t(\theta_t) \|_{\theta_t}^* \leq \frac{1}{4} $ for both barrier functions of \eqref{barrier} and \eqref{barrier2}.
\end{lemma}
\end{mdframed}

\begin{proof}

Let us first consider the barrier function in \eqref{barrier} for the difference-in-means testing.
We have
\begin{align}
\| \nabla \ell_t(\theta_t) \|_{\theta_t}^{*2}  & = 
(\nabla \ell_t(\theta_t) )^2 ( \nabla^2 R(\theta_t) )^{-1}
=\left(\frac{g_t}{1-g_t \theta_t}\right)^2 \frac{(1-\theta_t)^2(1+\theta_t)^2}{2+2\theta_t^2}\notag 
\\ &
\leq \frac{1}{(1- |\theta_t|)^2} \frac{(1-\theta_t)^2(1+\theta_t)^2}{2+2\theta_t^2} 
\leq \frac{1}{2+2\theta_t^2} \leq 1,
\label{eq:discuss_g}
\end{align}
where the second-to-least inequality is because $(1- |\theta_t|)^2$ will cancel out with either $(1-\theta_t)^2$
or $(1+\theta_t)^2$, depending on the sign of $\theta_t$. 

On the other hand, for the one-sided testing,
using \eqref{barrier2}, we have
\begin{align}
\| \nabla \ell_t(\theta_t) \|_{\theta_t}^{*2} = 
(\nabla \ell_t(\theta_t) )^2 ( \nabla^2 R(\theta_t) )^{-1}&
=\left(\frac{g_t}{1-g_t \theta_t}\right)^2 \frac{\theta_t^2(1-\theta_t)^2}{\theta_t^2+(1-\theta_t)^2}
\leq \frac{\theta_t^2}{\theta_t^2+(1-\theta_t)^2}
\leq 1,
\end{align}
where the second-to-last inequality uses that
$\left(\frac{g_t}{1-g_t \theta_t}\right)^2 \leq \frac{1}{(1-\theta_t)^2}$.
We have now completed the proof.
\end{proof}

We now identify a condition under which FTRL+Barrier achieves a smaller regret bound than ONS, with its implications for sequential hypothesis testing to be elaborated upon shortly.
\begin{mdframed}[backgroundcolor=black!10,rightline=false,leftline=false,topline=false,bottomline=false]
\begin{lemma} \label{lem:FTRL_const}
Denote $G_t:= \sum_{s=1}^t \nabla \ell(\theta_s) $. 
Suppose that there exists a time point $t_0$ such that for all $t \geq t_0$, we have
\begin{equation} \label{growth}
\textbf{(Linear growth of cumulative gradients)} \qquad 
\left | G_t   \right| \geq ct,
\end{equation}
for some $c>0$. 
Then, FTRL+Barrier (Algorithm~\ref{alg:FTRL}) satisfies  
$$\mathrm{Regret}_T(\theta_*) 
\leq \frac{t_0}{8\eta} + \frac{2}{c'^2\eta} \left( \frac{1}{t_0-1} - \frac{1}{T-1} \right )
+ \frac{  R(\theta_*)  }{\eta},$$
where $\eta\leq \frac{1}{4}$ is the parameter, $c' > 0$ is a constant, and $\theta_* \in K$ is any comparator.
\end{lemma}
\end{mdframed}

The proof of Lemma~\ref{lem:FTRL_const} is available in Section~\ref{supp:A}.
What Lemma~\ref{lem:FTRL_const} shows is that under the condition of 
a linear growth of cumulative gradients $G_t$, Algorithm~\ref{alg:FTRL} can actually
have a \emph{constant} regret $O(t_0)$, modulo the value of the barrier function at the benchmark $\theta_* \in \K$. In other words, once the total number of rounds $T$ is sufficiently large, i.e., $T \geq t_0$, the cumulative regret stays at $O(t_0)$, which is better than $O( \ln (T))$ of ONS (c.f. Lemma~\ref{lem:ONS}).
Hence, Algorithm~\ref{alg:Betting} with FTRL+barrier can have a shorter expected time to reject the null hypothesis $\H_0$ when $\H_1$ is true, as we will demonstrate shortly.  

In the following, we provide a couple of concrete scenarios of sequential hypothesis testing where the linear growth condition of cumulative gradients holds.

\begin{figure}[t]
\centering
\subfloat[Uniform distributions with disjoint supports. ]{\label{fig:a}\includegraphics[width=0.45\textwidth]{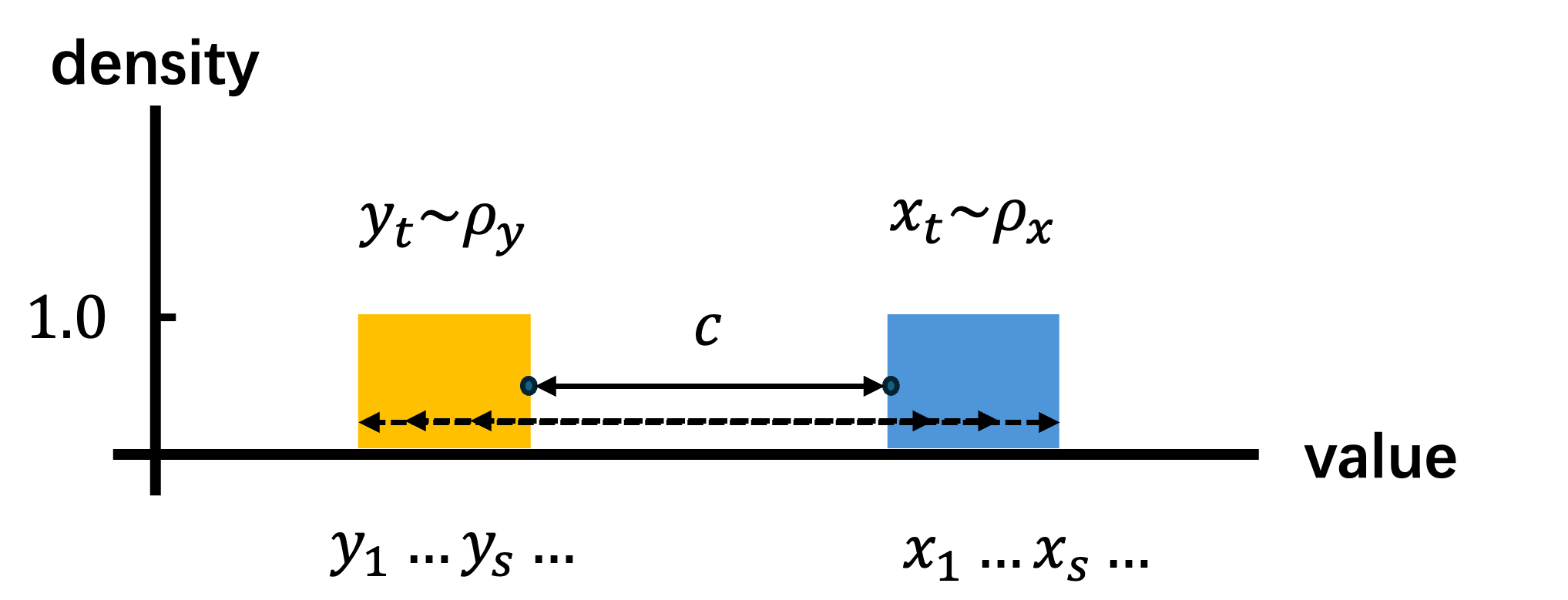}}
\subfloat[Distributions with overlapping supports.]{\label{fig:b}\includegraphics[width=0.45\textwidth]{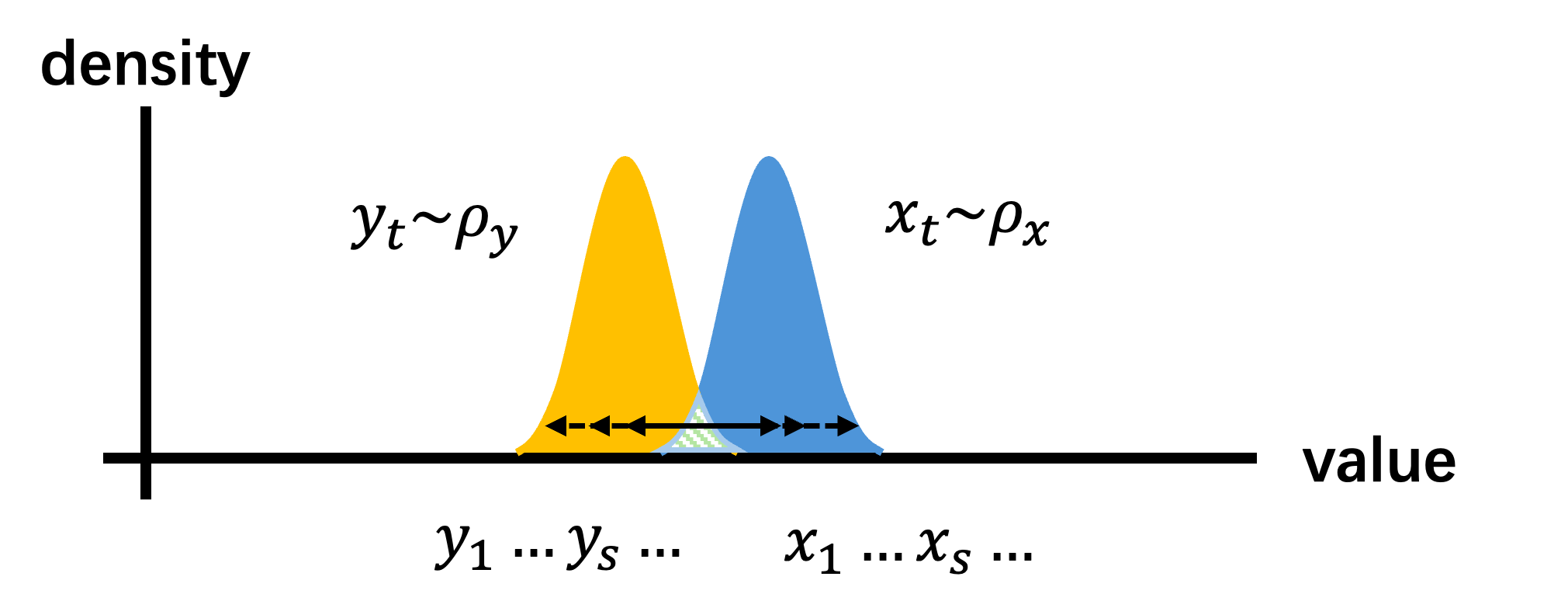}}
\caption{Illustration of Example 1 and 2.}
\label{fig:examples}
\end{figure}

\begin{mdframed}[backgroundcolor=black!10,rightline=false,leftline=false,topline=false,bottomline=false]
\noindent
\textbf{Example 1:}~
\textit{(Distributions with disjoint supports.) Consider $x_t \sim \rho_x$ and $y_t \sim \rho_y$, where $\rho_x$ and $\rho_y$ have disjoint but continuous supports, as illustrated in Figure~\ref{fig:a}. Then, the linear growth condition \eqref{growth} is satisfied for all $t \geq 1$.}
\end{mdframed}

\begin{proof} 
With the loss of generality, consider $x_t \in \mathrm{Uniform}(a,b)$ and $y_t \in \mathrm{Uniform}(m,n)$,
where $a> n$. Then, we have $g_s := x_s-y_s\geq a-n$. Thus, 
\begin{equation}
    \sum_{s=1}^t g_s =\sum_{s=1}^t (x_s-y_s) \geq (a-n)t\Rightarrow \left| \sum_{s=1}^t \nabla \ell_t(\theta_t) \right|=\left|\sum_{s=1}^t \frac{g_s}{1-g_s\theta_s}\right|
    \gtrsim (a-n)t.
    \label{eq:linear_cumulative}
\end{equation}
\end{proof}

\begin{mdframed}[backgroundcolor=black!10,rightline=false,leftline=false,topline=false,bottomline=false]
\noindent
\textbf{Example 2:}~
\textit{(Distributions with overlapping supports; high signal-to-noise ratio.) Denote $\sigma_x^2$ the variance of samples $\{x_t\}$ and $\sigma_y^2$ the variance of samples $\{y_t\}$. 
Then, with probability at least $1-\delta$, the linear growth condition of cumulative gradients \eqref{growth} holds for all $t \geq t_0$, where $t_0=\frac{1}{b^2\delta(\mu_x-\mu_y)^2/(\sigma_x^2+\sigma_y^2)}$, with $b\in(0,1)$. Specifically, the constant in \eqref{growth}, c, is proportional to $(1-b)|\mu_x-\mu_y|$. Based on the expression of $t_0$, if the signal-to-noise ratio (i.e., $\frac{(\mu_x-\mu_y)^2}{\sigma_x^2+\sigma_y^2}$) is high, then the growth condition is easily satisfied for any sufficiently large $t$.
Figure~\ref{fig:b} illustrates an example.}
\end{mdframed}

\begin{proof}
We denote the mean values as $\mathbb{E}(x_t)=\mu_x$ and $\mathbb{E}(y_t)=\mu_y$, their variances as $\mathrm{Var}(x_t)=\sigma_x^2$ and $\mathrm{Var}(y_t)=\sigma_y^2$.  Since $x_t$ and $y_t$ are independent random variables, then we have 
 $   \mathbb{E}\left(\sum_{s=1}^t g_s\right)=\sum_{s=1}^t\mathbb{E}(x_s-y_s)=t\cdot (\mu_x-\mu_y)$
and
 $   \mathrm{Var}\left(\sum_{s=1}^t g_s\right)=\sum_{s=1}^t\mathrm{Var}(x_s-y_s)=t\cdot (\sigma_x^2+\sigma_y^2).$

Let us assume that $\mu_x>\mu_y$ without the loss of generality. 
In the following, we would like to show with high probability,
the linear growth condition holds, i.e.,
$\sum_{s=1}^t g_s \geq (1 - b) t ( \mu_x - \mu_y)$
for some $b \in (0,1)$.
By the Chebyshev's inequality, the probability of the complementary event when the linear growth condition
does not hold can be bounded as:
\begin{align}
    \mathbb{P}\left(\mathbb{E}\left[\sum_{s=1}^t g_s\right]-\sum_{s=1}^t g_t\geq b\cdot\mathbb{E}\left[\sum_{s=1}^t g_s\right]\right)&\leq \mathbb{P}\left(\left|\mathbb{E}\left[\sum_{s=1}^t g_s\right]-\sum_{s=1}^t g_s\right|\geq b \mathbb{E}\left[ \sum_{s=1}^t g_s\right] \right)\notag\\
    &\leq\frac{\mathrm{Var}(\sum_{s=1}^t g_s)}{b^2 (\mathbb{E}[\sum_{s=1}^t g_s])^2}
    =\frac{1}{t}\frac{1}{b^2}\frac{1}{(\mu_x-\mu_y)^2/(\sigma_x^2+\sigma_y^2)}\label{eq:t0_highpro}.
\end{align}

When $t$ is large enough, and the sum of the variance is much smaller than the difference of the mean values (i.e., high signal-to-noise ratio), the probability
that the linear growth condition does not hold
can be small. Specifically, by \eqref{eq:t0_highpro},
 if we would like the failure probability be bounded by $\delta$, then for $t\geq \frac{1}{b^2\delta(\mu_x-\mu_y)^2/(\sigma_x^2+\sigma_y^2)}$, we know that the linear growth condition holds
with probability at least $1-\delta$.

\end{proof}

\begin{mdframed}[backgroundcolor=black!10,rightline=false,leftline=false,topline=false,bottomline=false]
\begin{theorem} \label{thm:main1}
Following the notations and assumptions in Theorem~\ref{thm:1}.
Let $\mathrm{OAlg}$ be FTRL+Barrier (Algorithm~\ref{alg:FTRL}) with the parameter $\eta\leq \frac{1}{4}$. Then Algorithm~\ref{alg:Betting} is a level-$\alpha$ sequential test with asymptotic power one. Furthermore, assume that the linear growth condition \eqref{growth} holds for all $t \geq t_0$ for some $t_0$. Then,
the expected rejection time $\tau$ under $\H_1$ can be bounded as
$$ \E[ \tau ] = \Theta\left( \frac{t_0+\ln \left(1/\alpha\right)}{ \omega_* } + \frac{\sigma^2}{\omega_*^2} \right).$$ 
\end{theorem}
\end{mdframed}

\begin{proof}
We first note that by invoking Theorem~\ref{thm:1},
Algorithm~\ref{alg:Betting} is a level-$\alpha$ sequential test with asymptotic power one.

To derive the expected rejection time $\E[\tau]$,
by Theorem~\ref{thm:1},
we need to find $t_*$ so that
$\mathrm{Regret}_t(\theta_*) \leq \frac{\omega_* t}{2} - \ln \left(1/\alpha\right)$
for all $t \geq t_*$.
By Lemma~\ref{lem:FTRL_const}, if the linear growth condition holds for all $t \geq t_0$, then $\mathrm{Regret}_t(\theta_*)= \Theta(t_0)$.
Consequently, the requirement of the regret becomes
$t_0 \lesssim \frac{\omega_* t}{2} - \ln \left(1/\alpha\right)$, which translates to
$t \gtrsim  \frac{t_0+\ln \left(1/\alpha\right)}{ \omega_* }$.
Therefore, $t_* = \Theta\left( \frac{t_0+\ln \left(1/\alpha\right)}{ \omega_* } \right)$.
Then, by Theorem~\ref{thm:1}, we have
$\E[ \tau ] \leq t_* + \frac{2 \sigma^2}{ \omega_*^2 }
= \Theta\left( \frac{t_0+\ln \left(1/\alpha\right)}{ \omega_* } + \frac{\sigma^2}{\omega_*^2} \right)$,
which completes the proof.
\end{proof}

We note that recently, Waudby-Smith et al.~\cite{waudby2025universal} and Agrawal and Ramdas~\cite{agrawal2025stopping} independently have
very nice results about the lower bound of the expected rejection time when the data is i.i.d.
Specifically, Waudby-Smith et al.~\cite{waudby2025universal} show that 
when the significance-level parameter $\alpha$ approaches zero, it holds that
\begin{equation} \label{lower}
\lim_{ \alpha \to 0_+} \frac{ \E[\tau] }{  \ln (1/\alpha) } \geq \frac{1}{\omega_*},
\end{equation}
where $\omega_*$ is the maximum of the expected wealth in a single round. The upper bound result of 
Theorem~\ref{thm:main1} matches that of \eqref{lower}  when
$\alpha$ approaches $0$.
On the other hand, Agrawal and Ramdas~\cite{agrawal2025stopping} provide two lower-bound results. One of them also concerns the case when $\alpha$ approaches zero: they show that the ratio $\lim_{\alpha \to 0_+} \frac{ \E[\tau] }{ \ln (1/\alpha) }$ is lower-bounded by the inverse of a certain infimum KL divergence regarding testing the null $\mathcal{P}$ against the alternative $\mathcal{Q}$. The other concerns the case when the KL divergence approaches $0$ while fixing $\alpha$, for which we refer the reader to their work for further details.

We now compare the result of Theorem~\ref{thm:main1} for FTRL+Barrier and the guarantee of using ONS (see Algorithm~\ref{alg:ONS1})
as $\mathrm{OAlg}$ in Algorithm~\ref{alg:Betting}
for testing by betting in the literature.  
Chugg et al.~\cite{Chugg2023} provide an upper-bound of the expected rejection time for ONS 
in the difference-in-means setting, which is 
$\E[ \tau^{\mathrm{ONS}} ] \lesssim \frac{1}{\Delta^2} \ln \left( \frac{1}{\Delta^2 \alpha}  \right) $,
where $\Delta:= |\mu_x - \mu_y |$.
However, 
it is not straightforwardly compare the bound in Theorem~\ref{thm:main1} and the result in \cite{Chugg2023} without approximating $\omega_*$,
which in turn might need some assumptions on $\theta_* := \arg\max_{\theta \in \K } \E\left[ \ln (1 - g \theta) \right]$.
To this end, let us consider $| \theta_*| \leq \frac{1}{2}$.
Then, by using the inequality $\ln (1+c) \geq c - c^2, \forall c \in [-1/2,1/2]$, 
we can obtain
\begin{align}
\omega_* = 
\max_{ \theta \in K }
\E\left[ \ln \left(1-g \theta\right) \right] 
& \geq 
\max_{ \theta \in K }
\E\left[ - \theta g - \theta^2 g^2     \right]
 =
\max_{ \theta \in K }  - \theta (\mu_x - \mu_y) - \theta^2 
\left( \mathrm{Var}[g] + (\mu_x - \mu_y )^2    \right) \notag 
\\ & 
\approx \frac{ \Delta^2 }{ 
\mathrm{Var}[ g ] + \Delta^2 }
\approx  \frac{ \Delta^2 }{ 
1 + \Delta^2 }, \label{approxi}
\end{align}
where we note $\E[ g ] = \mu_x - \mu_y$,
$\mathrm{Var}[g]= \E[ g^2 ] - (\E[ g ])^2 $,
and $\mathrm{Var}[ g] \leq 1$.
Hence, if $\frac{1}{\omega_*} = \Theta\left( \frac{1}{\Delta^2}  \right)$, then when $\alpha$ approaches $0$,
the rejecting time bound of FTRL+Barrier
becomes $\Theta\left( 
\frac{1}{ \Delta^2 } 
\ln \left( \frac{1}{\alpha} \right) \right)$
under the linear growth condition, which is better than
 that using ONS, i.e., better than
 $\Theta \left( \frac{1}{\Delta^2} \ln \left( \frac{1}{\Delta^2 \alpha}  \right) \right)$.

\subsection{Optimistic-FTRL+Barrier for Testing by Betting} \label{section:main2}

In this subsection, we introduce another algorithm for testing by betting, which we call Optimistic-FTRL+Barrier. 
We draw on ideas from \emph{optimistic online learning}, which concerns the scenario where the learner incorporates a guess of the next gradient $m_t$ to determine the action $\theta_t$ at each round before observing its loss function \cite{chiang2012online, rakhlin2013online, joulani2017modular,wang2018acceleration,orabona2019modern,wang2024no,chen2024optimistic}. 
In particular, if the sequence of loss functions is predictive and the learner's estimation of them is accurate enough, then smaller regret can be achieved. We extend this machinery to testing by betting, and the intuition is that if the player can predict their incoming sequence of samples, they should incorporate this information into their bets to enhance their performance in 
the testing-by-betting game.
This variant is shown in Algorithm~\ref{alg:OFTRL}, 
where we note that $R(\theta)$ is the barrier function, which is 
\eqref{barrier} for the difference-in-means testing 
and \eqref{barrier2} for the one-sided testing.

\begin{algorithm}[t]
\caption{OptimisticFTRL+Barrier as $\mathrm{OAlg}$ for Testing by Betting}
\label{alg:OFTRL}
\begin{algorithmic}[1] 
\small
\STATE \textbf{Init:} $\theta_1 \in \K$.
\STATE \textbf{Require:} the parameter $\eta$. 
\STATE \textbf{Specify:} the decision space $\K$ and the barrier function $R(\theta)$.
\STATE\textbf{For} $t=1,2,\dots, T$ \textbf{do}
\STATE\quad Play $\theta_t \in \K$.
\STATE\quad Receive $\ell_t(\theta)= - \ln (1-g_t \theta)$ and incur loss $\ell_t(\theta_t)$.
\STATE \quad Construct a guess $m_{t+1}$ of the next gradient.
\STATE \quad Update
$\theta_{t+1} \gets \underset{\theta \in \K }{\arg\min} \, \eta 
\left \langle  m_{t+1} + \sum_{s=1}^t \nabla \ell_s(\theta_s), \theta \right  \rangle + 
 R(\theta).$
\STATE\textbf{End For}
\end{algorithmic}
\end{algorithm}

\begin{mdframed}[backgroundcolor=black!10,rightline=false,leftline=false,topline=false,bottomline=false]
\begin{lemma} [Adapted from \cite{wang2024no}] \label{lem:OptimisticFTRL}
Optimistic-FTRL+Barrier (Algorithm~\ref{alg:OFTRL}) has
\begin{equation}
\mathrm{Regret}_T(\theta_*)
\leq
2\eta\sum_{t=1}^T\|\nabla \ell_t (\theta_t)- m_t \|_{\theta_t}^* \|\nabla\ell_t (\theta_t)\|_{\theta_t}^*+ \frac{R( \theta_*) - \min_{\theta \in \K} R(\theta)}{\eta} ,
\end{equation}
where $\theta_* \in \mathcal{K}$ is any comparator.
\end{lemma}
\end{mdframed}

\begin{wrapfigure}[11]{R}{0.44\textwidth}
    \centering
    \includegraphics[width=0.44\textwidth]{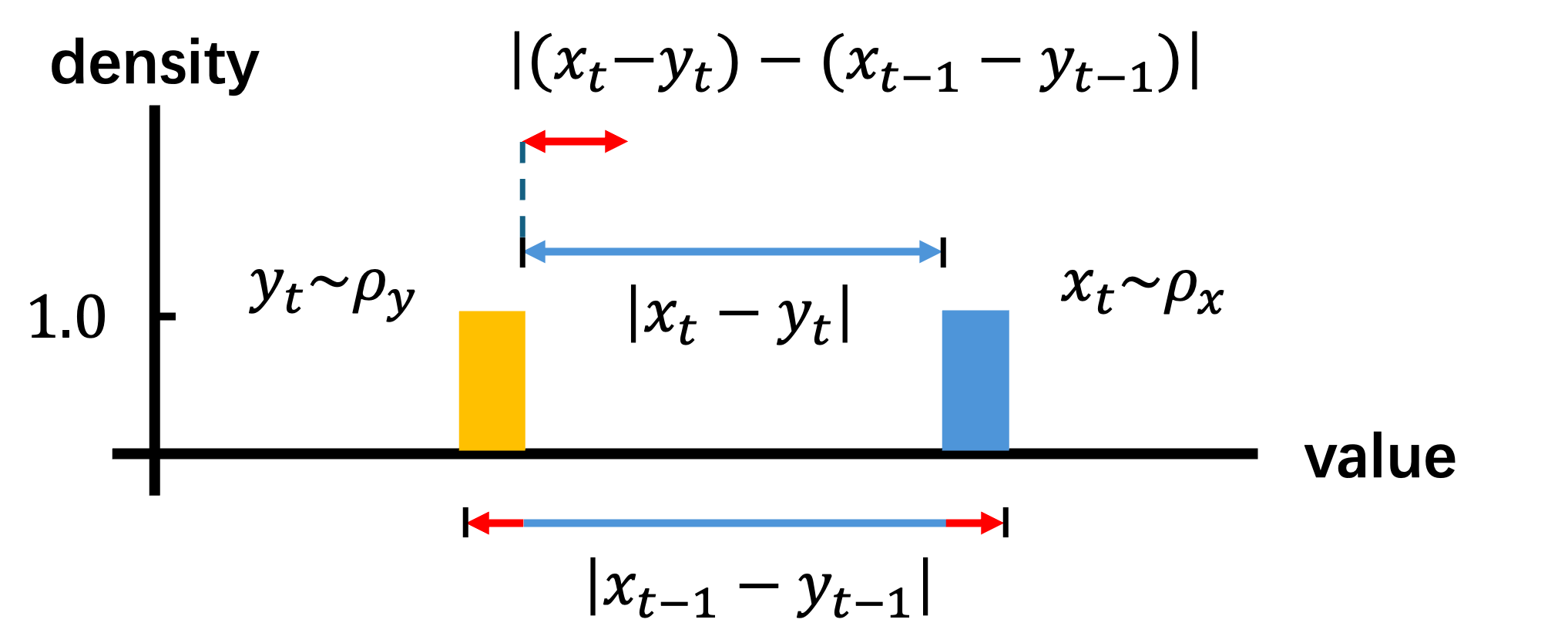}
    \caption{Illustration of Example 3. 
    }
    \label{fig:ex4}
\end{wrapfigure}

Comparing the regret bound above and the regret of FTRL+Barrier in Lemma~\ref{lem:FTRL}, it is evident that if the guess $m_t$ is close to the next gradient $\nabla \ell_t(\theta_t)$, 
i.e., when $\clubsuit: \|\nabla \ell_t (\theta_t)- m_t \|_{\theta_t}^* \ll \|\nabla \ell_t (\theta_t) \|_{\theta_t}^* $,
then the \emph{optimistic} version has a smaller regret, thereby speeding up the process of rejecting the null when the alternative holds. However, while the regret can be potentially smaller than its non-optimistic counterpart when $m_t$ is ``close" to $\nabla \ell_t (\theta_t)$.
One natural choice for $m_{t+1}$ is to set it to the most recent gradient $\nabla \ell_{t}(\theta_{t})$ at $t$, i.e., set $m_{t+1} \gets \nabla \ell_{t}(\theta_{t})$. 
The hope is that the gradient in the testing-by-betting game changes slowly, i.e.,
$\nabla \ell_{t}(\theta_{t}) \approx \nabla \ell_{t-1}(\theta_{t-1}),$
in which case the \emph{optimistic} variant would have a real advantage. However, the local norm $\|\nabla \ell_t (\theta_t)- \nabla \ell_{t-1} (\theta_{t-1}) \|_{\theta_t}^*$ may not be bounded. 
Determining how to choose the hint $m_{t+1}$ such that the local norm remains bounded while achieving smaller regret bounds when $m_{t+1}$ provides a good estimate remains an open question in this work.

On the other hand, Example 3 below illustrates a scenario in difference-in-means testing where Optimistic-FTRL+Barrier can perform better than its non-optimistic counterpart under certain conditions.
Here, we use the fact that
$\| \nabla \ell_t (\theta_t)- \nabla \ell_{t-1} (\theta_{t-1}) \|_{\theta_t}^{*2}
\leq \frac{ \left( \theta_t^2 (g_t - g_{t-1})^2 + g_{t-1}^2 (\theta_t - \theta_{t-1})^2 \right) (1 - \theta_t^2)^2 }{ (1+\theta_t^2)(1 - g_t \theta_t )^2 (1 - g_{t-1} \theta_{t-1})^2 }$
and the expression of $\| \nabla \ell_t (\theta_t) \|_{\theta_t}^{*2}$ in \eqref{eqnorm}.

\begin{mdframed}
\noindent
\textbf{Example 3:}~\textit{(IID Samples from distributions with small variances.) 
Consider $x_t$ and $y_t$ are respectively drawn from two uniform distributions with disjoint supports, i.e., $x_t\sim\rho_x$ and $y_t\sim\rho_y$, where $\rho_x$ and $\rho_y$ have small variances. 
When the update and the environment changes slowly such that $\theta_t \approx \theta_{t-1}$, then condition $\clubsuit$ could potentially hold if
\begin{equation}
    |{g_t}-{g_{t-1}}| \ll |{g_{t}}|\Leftrightarrow |(x_t-y_t)-(x_{t-1}-y_{t-1})| \ll |x_t-y_t|, \label{eq:op_better_eqt3}
\end{equation}
which can be satisfied when the distance between the means of two distributions is significantly larger than their variances, as illustrated in Figure~\ref{fig:ex4}.}
\end{mdframed}

\section{Experiments}

\noindent
\textbf{Baselines.} In the following experiments, we compare our proposed methods FTRL+Barrier (Algorithm~\ref{alg:FTRL}) and OFTRL+Barrier (Algorithm~\ref{alg:OFTRL}) with three baseline methods, which we detail as follows.
For all the following simulation results, we use the same parameter value, $\eta = 1$, for both FTRL+Barrier and OFTRL+Barrier.

\begin{wrapfigure}[12]{R}{0.46\textwidth}
\vspace*{-8mm}
    \begin{minipage}{0.46\textwidth}
\begin{algorithm}[H]
\caption{Online Newton Step.}
\label{alg:ONS1}
\begin{algorithmic}[1] 
\STATE Set decision space $\K^{\mathrm{ONS}}$.
\STATE Init $a_0 \gets 1$.
\STATE \textbf{For} $t = 1, 2, \dots, T$ \textbf{do}
    \STATE \quad Play $\theta_t \in \K^{\mathrm{ONS}}$.
    \STATE \quad Receive $\ell_t(\theta) := - \ln (1 - g_t \theta)$ \\
    \qquad and incur loss $\ell_t(\theta_t)$.
    \STATE \quad Set $b_t = \frac{g_t}{1 - g_t \theta_t}$,
    \STATE \quad Update $a_t = a_{t-1} + b_t^2$.
    \STATE \quad 
    $\theta_{t+1} = \operatorname{Proj}_{\K^{\mathrm{ONS}}}\left[ \theta_t - \frac{2}{2 - \ln 3} \frac{b_t}{a_t} \right]$. \\
\textbf{End For}
\end{algorithmic}
\end{algorithm}
  \end{minipage}
\end{wrapfigure}

\noindent
\textbf{1. Online Newton Step (ONS)}.
This is perhaps the most popular online learning algorithm in testing by betting game \cite{Cutkosky2018, Chugg2023, podkopaev2023sequential, teneggi2024bet, waudby2024estimating,podkopaev2024sequential,shekhar2023nonparametric,pandeva2024deep, CW2025}. 
As described in Algorithm~\ref{alg:ONS1}, ONS has a closed-form update and hence the implementation cost per round is cheap. 
We note that $\operatorname{Proj}_{\K^{\mathrm{ONS}}}$ in Line 8
is the projection operator on the corresponding constraint set $\K^{\mathrm{ONS}}$ described below.
Furthermore, since the loss functions are exp-concave in the games of testing-by-betting, ONS has a guarantee of $O\left(\log (T)\right)$ regret bound \cite{orabona2019modern}. However, as we pointed out in the introduction, to avoid the gradient explosion issue of the log loss, the aforementioned works consider the heuristic of halving the decision space: 
\begin{align}
\textbf{Difference-in-Means Testing (c.f.~\eqref{scenario-1a}) }:
& \qquad 
\mathcal{K}^{\mathrm{ONS}} \gets [-1/2,1/2] \notag \\ 
\textbf{One-Sided Testing  (c.f.~\eqref{scenario-2a}) }:
& \qquad
\mathcal{K}^{\mathrm{ONS}} \gets [0,1/2]. \notag
\end{align}
We also recall that $g_t := x_t - y_t \in [-1,1]$ in 
difference-in-means testing and $g_t := \mu_0 - x_t \in [-1,1]$ 
in one-sided testing.

\noindent
\textbf{2. Online Portforlio Algorithm based on Cover and Ordentlich \cite{cover1996universal} (CO96)}.
This algorithm is also considered in \cite{waudby2025universal}. It rejects the null hypothesis $H_0$ at time $t$ when the quantity $W_t^{\mathrm{CO96}}$ satisfies $W_t^{\mathrm{CO96}} \geq \frac{1}{\alpha}$, where $W_t^{\mathrm{CO96}}$ is defined as:
\begin{equation}
W_t^{\mathrm{CO96}} 
 = \exp\left\{ \ln \left( \hat{W}_t(\theta_t^{\max}) \right) - \frac{1}{2} \ln(t + 1) - \ln 2 \right\},
\end{equation}
where \footnote{
Instead of specifying 
$\ln \left( \hat{W}_t(\theta) \right) :=
\sum_{s=1}^t \ln \left(1 - (\mu_0 - x_s) \theta \right)$ 
for one-sided testing, we strictly follow \cite{waudby2025universal}, which considers \eqref{eq:woudby}. We refer the reader to \cite{waudby2025universal} for details of this formulation, which is related to their notion of portfolio regret.
}
\begin{align}
 \ln \left( \hat{W}_t(\theta) \right) := 
\begin{cases}
\sum_{s=1}^t \ln \left(1- (x_s-y_s) \theta \right)   & \textbf{(Difference-in-Means Testing)}
\\ 
\sum_{s=1}^t \ln \left(1- (\mu_0-x_s) \frac{\theta}{\mu_0}  \right)   & \textbf{(One-Sided Testing)} \label{eq:woudby},
\end{cases}
\end{align}
and that
$\theta_t^{\max} := \arg\max_{\theta \in \K^{\mathrm{OPS}}}
\ln \left( \hat{W}_t(\theta) \right)$, where $\K^{\mathrm{OPS}} = [-1,1]$ for difference-in-means testing, while $\K^{\mathrm{OPS}} = [0,1]$ for one-sided testing.
We note that computing $W_t^{\mathrm{CO96}}$
requires solving the optimization problem to obtain
$\theta_t^{\max}$.
Hence, this method is computationally more expensive than ONS, 
as ONS has a closed-form update.

\noindent
\textbf{3. Online Portfolio Algorithm based on Orabona and Jun \cite{orabona2023tight} (OJ23)}.
This is an algorithm also considered in \cite{waudby2025universal}.
It rejects the null hypothesis $H_0$ at time $t$ when the quantity $W_t^{\mathrm{OJ96}}$ satisfies $W_t^{\mathrm{CO96}} \geq \frac{1}{\alpha}$, where $W_t^{\mathrm{OJ23}}$ is defined as:

\begin{equation}
W_t^{\text{OJ23}} = \exp\left\{ \ln \left( \hat{W}_t(\theta_t^{\max}) \right) - \max_{j=0,\ldots,t} \ln \left( \frac{\pi\cdot (\lambda_t^{\max})^j (1 - \lambda_t^{\max})^{t-j} \Gamma(t+1)}{\Gamma(j+1/2)\Gamma(t-j+1/2)} \right) \right\}.
\end{equation}
As with CO96, this method involves solving the same optimization problem to get $\theta_t^{\max}$.

\subsection{Experimental Results of Difference-in-Means Testing} 

We conduct a few simulations to evaluate our proposed methods and the baselines detailed in the previous subsection.
The first concerns an \emph{easy} setting illustrated in Figure~\ref{fig:ex4} for Example~3. 
More precisely, the $H_1$ scenario is that $x_t \sim \mathrm{Uniform}(0.2, 0.4)$ 
and $y_t \sim \mathrm{Uniform}(0.7, 0.9)$.
Figure~\ref{Result_1} reports the simulation results. 
Here, all the methods use the same parameter value $\alpha=0.05$.
OFTRL+Barrier outperforms its non-optimistic counterpart, which is consistent with the analysis in Example~3.
In this setting, OJ23 and OFTRL+Barrier are the most competitive methods. 
However, OFTRL+Barrier has the additional advantage of a low implementation cost over OJ23, thanks to its closed-form update.
We also note that both FTRL+Barrier and OFTRL+Barrier outperform ONS. 

\begin{figure}[h]
\centering
\includegraphics[width=0.8\textwidth]{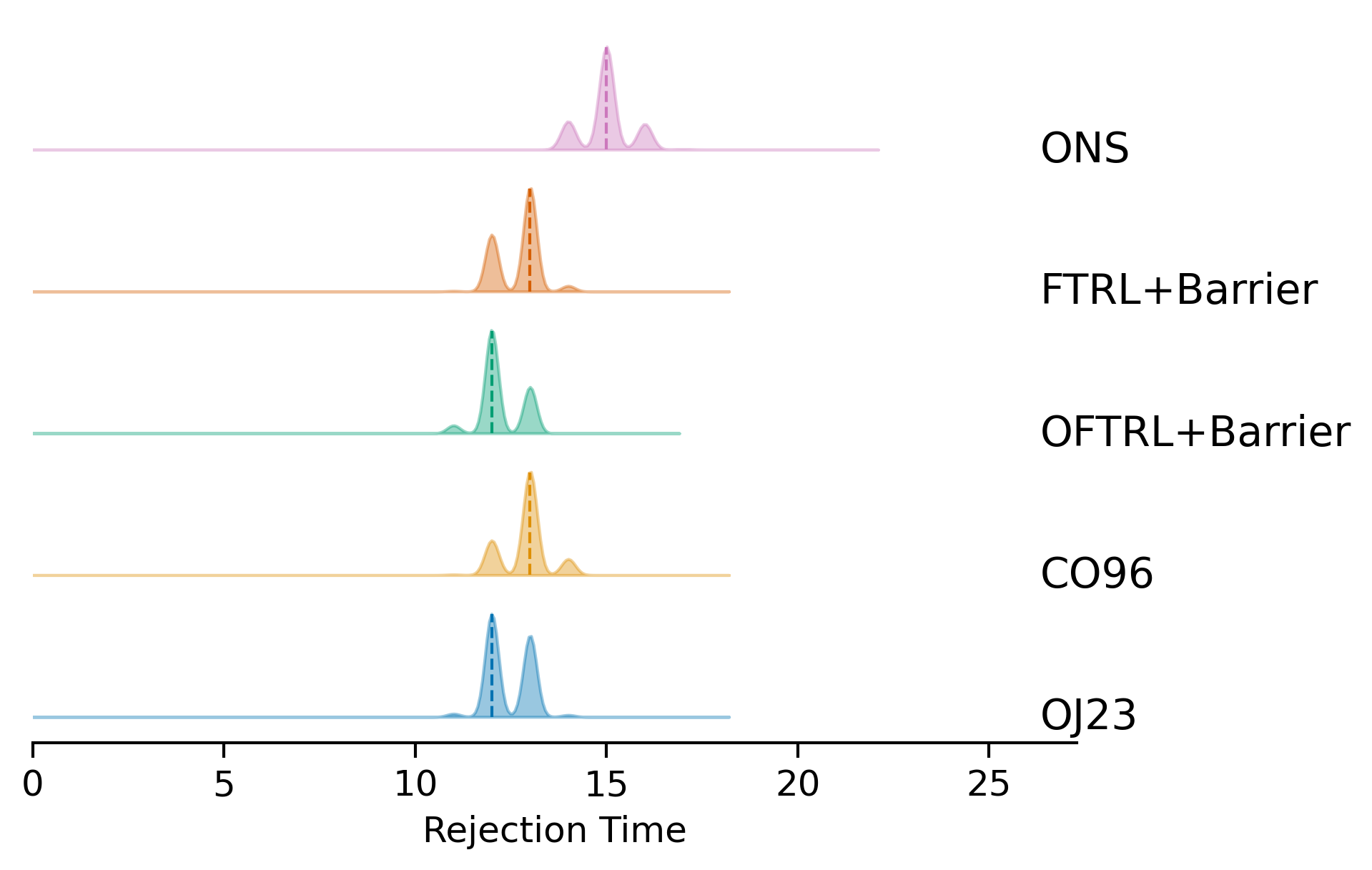}
\caption{Rejection time of different methods under the ``easy" regime illustrated in Figure~\ref{fig:ex4}.}
\label{Result_1}
\end{figure}

We now switch to more difficult regimes.  
For this purpose, we consider two types of distributions:  
uniform distributions and truncated normal distributions.  
Figure~\ref{scenario} illustrates the normalized histograms of samples from these distributions.  
One can see that the distributions have substantial overlap,  
which potentially makes the task of sequential hypothesis testing more challenging.
More concretely, we consider two $H_1$ scenarios:  
(a) $x_t \sim \mathrm{Uniform}(0.2, 0.8)$ and $y_t \sim \mathrm{Uniform}(0.3, 0.9)$;  
(b) $x_t \sim \hat{N}(0.5, 0.15)$ and $y_t \sim \hat{N}(0.65, 0.15)$,  
where $\hat{N}(\mu, \sigma^2)$ denotes a truncated normal distribution centered at $\mu$ with variance $\sigma^2$. 
To simulate the $H_0$ scenarios, we run all algorithms for $T = 500$ steps and record whether they reject $H_0$ within the first $T = 500$ steps.  
The sequence of samples $\{y_t\}$ in the $H_0$ scenario is obtained from those in the $H_1$ scenario by shifting all values by the same constant so that the first $T = 500$ samples of both sequences of $\{x_t\}$ and $\{y_t\}$ for the $H_0$ scenario have the same empirical mean.

\begin{figure}[t]
\centering
\includegraphics[width=0.6\textwidth]{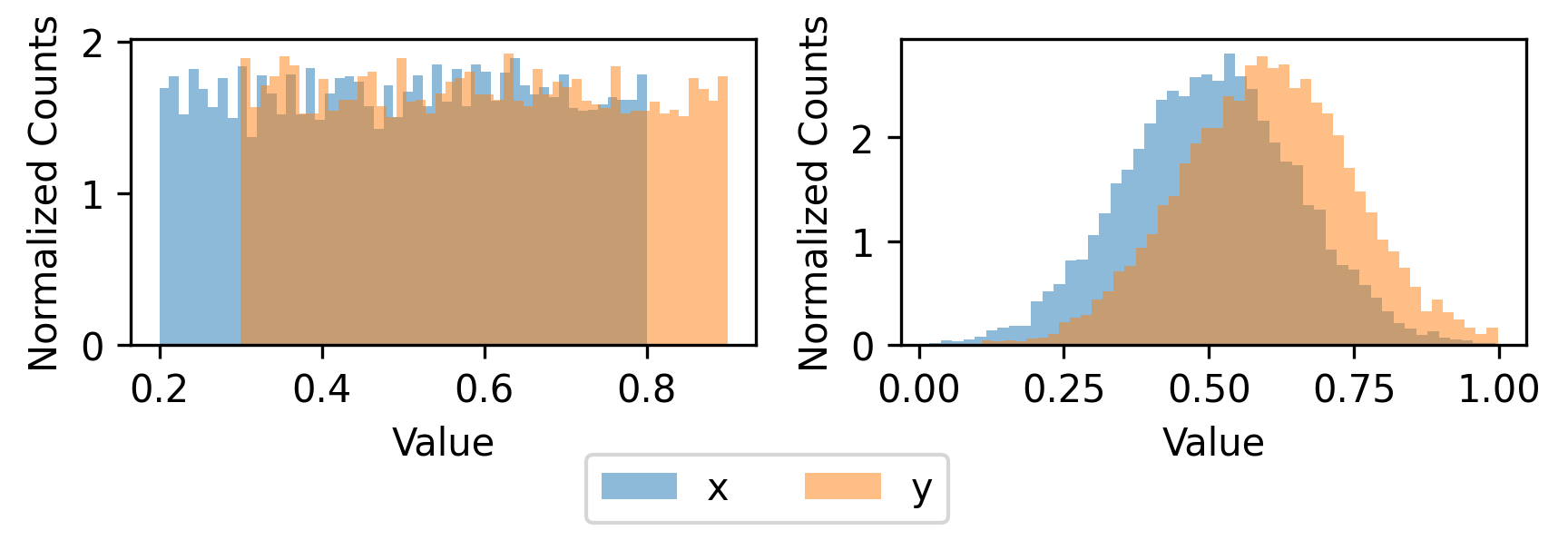}
\caption{\textbf{Empirical distributions in the simulation.} (Left) Uniform distributions; (Right) truncated normal distributions. The histograms are normalized so that the area under each curve is equal to one.}
\label{scenario}
\end{figure}

\begin{figure}[t]
\centering
\includegraphics[width=0.8\textwidth]{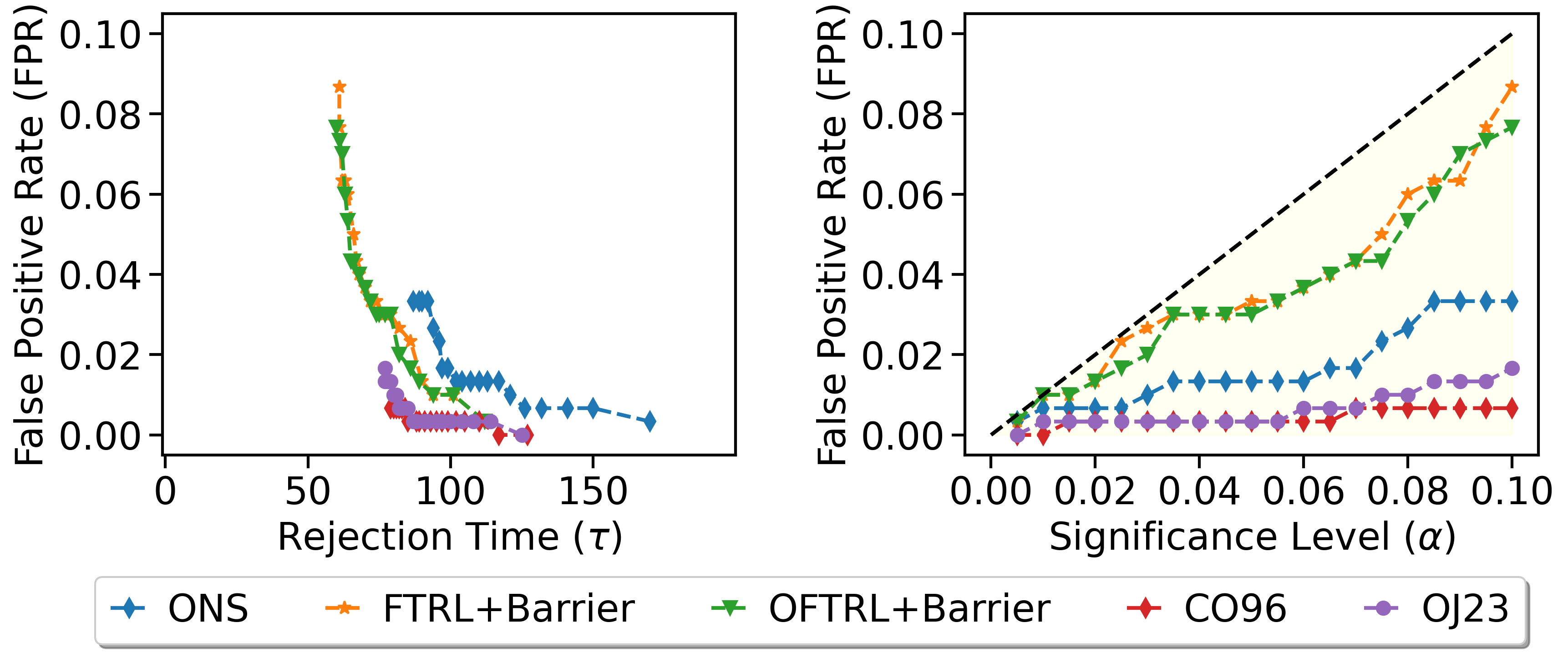}
\caption{Scenario (a) of difference-in-means testing: substantially overlapped uniform distributions.}
\label{exp:difference-in-means-uniform}
\end{figure}

Figures~\ref{exp:difference-in-means-uniform} and~\ref{exp:difference-in-means-truncated-normal} present the results for the aforementioned difference-in-means experimental setup under uniform and truncated normal distributions, respectively. To ensure statistical robustness, each method 
using a fixed parameter configuration was evaluated across $300$ repeated runs, with all reported results representing averages computed over these repeated runs.
Each individual marker along each curve denotes a method with a specific value of the significance level parameter $\alpha$. 
Specifically, we examine $20$ uniformly distributed values of $\alpha$ spanning the interval between $[0.005,0.1]$.
The subfigure on the left reports the rejection time under $H_1$ and the empirical rate of erroneously rejecting $H_0$ when $H_0$ is true by $T=500$ steps. 
The subfigure on the right supports that each method is a valid level-$\alpha$ test, as the empirical false positive rate is bounded by the corresponding value of~$\alpha$.
Both figures show that our proposed FTRL+Barrier and OFTRL+Barrier methods outperform ONS, which supports the motivation of this work.  
Furthermore, our proposed methods are also competitive with CO96 and OJ23, despite the latter involving a more complex implementation.  
In particular, the results indicate that FTRL+Barrier and OFTRL+Barrier can reject the null faster, albeit at the cost of a little higher empirical false positive rate.  
Nevertheless, this rate remains well controlled by the significance-level parameter~$\alpha$, as confirmed by the subfigure on the right.

\begin{figure}[t]
\centering
\includegraphics[width=0.8\textwidth]{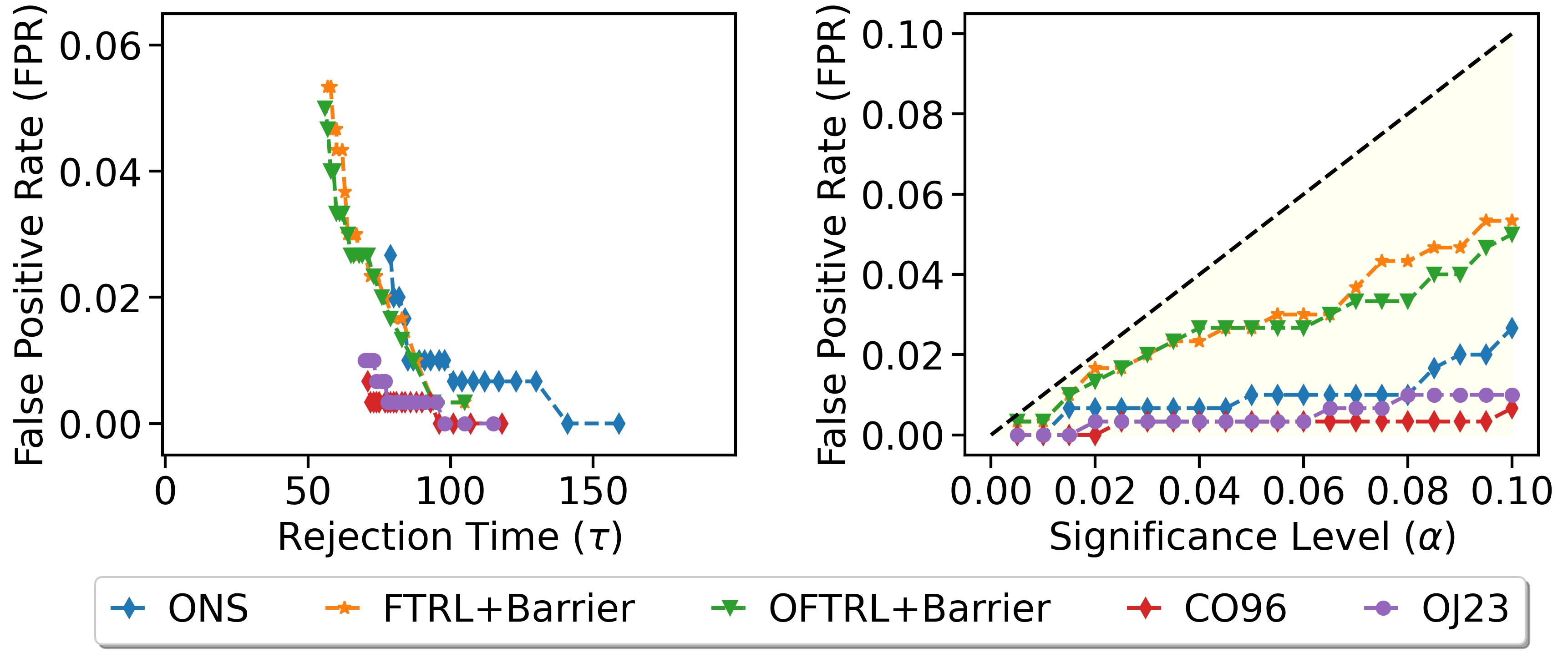}
\caption{Scenario (b) of difference-in-means testing: substantially overlapped truncated normal distributions.}
\label{exp:difference-in-means-truncated-normal}
\end{figure}

\begin{table}[h]
\centering
\caption{Runtime per iteration runtime in milliseconds for difference-in-means testing.}
\label{tab:avg-iteration-runtime1}
\begin{tabular}{lcc}
\toprule
\textbf{Method} & \textbf{Average (ms)} & \textbf{Std. Dev. (ms)} \\
\midrule
FTRL+Barrier   & 0.001694 & 0.000487 \\
OFTRL+Barrier  & 0.002280 & 0.000487 \\
ONS            & 0.001146 & 0.000391 \\
CO96           & 0.381144 & 0.108137 \\
OJ23           & 0.600817 & 0.185730 \\
\bottomrule
\end{tabular}
\end{table}

\noindent
\textbf{Implementation cost per round.}
Compared to the baselines CO96 and OJ23, a notable feature of FTRL+Barrier and OFTRL+Barrier is their closed-form updates.  
Tables~\ref{tab:avg-iteration-runtime1} report the average runtime in milliseconds per round of sequential hypothesis testing, where simulations were conducted on a MacBook Pro. FTRL+Barrier, OFTRL+Barrier, and ONS have significantly lower implementation costs compared to CO96 and OJ23. In particular, our methods achieve approximately a two-order-of-magnitude speedup relative to CO96 and OJ23.  
Although ONS exhibits the lowest average runtime, FTRL+Barrier and OFTRL+Barrier remain lightweight, as each iteration involves solving only a quadratic equation.  
Furthermore, Figure~\ref{Result_1},~\ref{exp:difference-in-means-uniform}, and~\ref{exp:difference-in-means-truncated-normal} have demonstrated that both FTRL+Barrier and OFTRL+Barrier outperform ONS in terms of rejection time and the balance between the performance under $H_0$ and $H_1$.

\subsection{Experimental Results of One-Sided Testing}

In this subsection, we report our simulation results for one-sided testing.  
Following \cite{waudby2025universal}, we consider two scenarios.
\textbf{Scenario (a)}: $H_0: \mu_x < 0.1$, where the samples $x_t$ are drawn from a Bernoulli distribution with parameter $0.95$ to simulate $H_1$, whereas to simulate $H_0$, the samples $x_t$ are drawn from a Bernoulli distribution with parameter $0.09$.
\textbf{Scenario (b)}: $H_0: \mu_x < 0.3$, where the samples $x_t$ are drawn from a Bernoulli distribution with parameter $0.4$ to simulate $H_1$, whereas to simulate $H_0$, the samples $x_t$ are drawn from a Bernoulli distribution with parameter $0.29$.
Figures~\ref{exp:one-sided-easy} and~\ref{exp:one-sided-hard} report the results,  
which show that the proposed FTRL+Barrier and OFTRL+Barrier consistently outperform ONS.
While CO96 and OJ23 perform well under the easy hypothesis testing scenario (Figure~\ref{exp:one-sided-easy}),  
our proposed methods outperform these baselines under the more challenging setting (Figure~\ref{exp:one-sided-hard}).

We also report the average runtime per round with standard deviation in Table~\ref{tab:avg-iteration-runtime2},  
where the results correspond to the challenging setting (i.e., the task in Figure~\ref{exp:one-sided-hard}).  
Similar runtimes were observed in the setting considered in Figure~\ref{exp:one-sided-easy}.
Table~\ref{tab:avg-iteration-runtime2} shows that our methods achieve approximately a two-order-of-magnitude speedup relative to CO96 and OJ23 in the one-sided testing, which is consistent with the findings in Table~\ref{tab:avg-iteration-runtime1}.

\begin{figure}[t]
\centering
\includegraphics[width=0.8\textwidth]{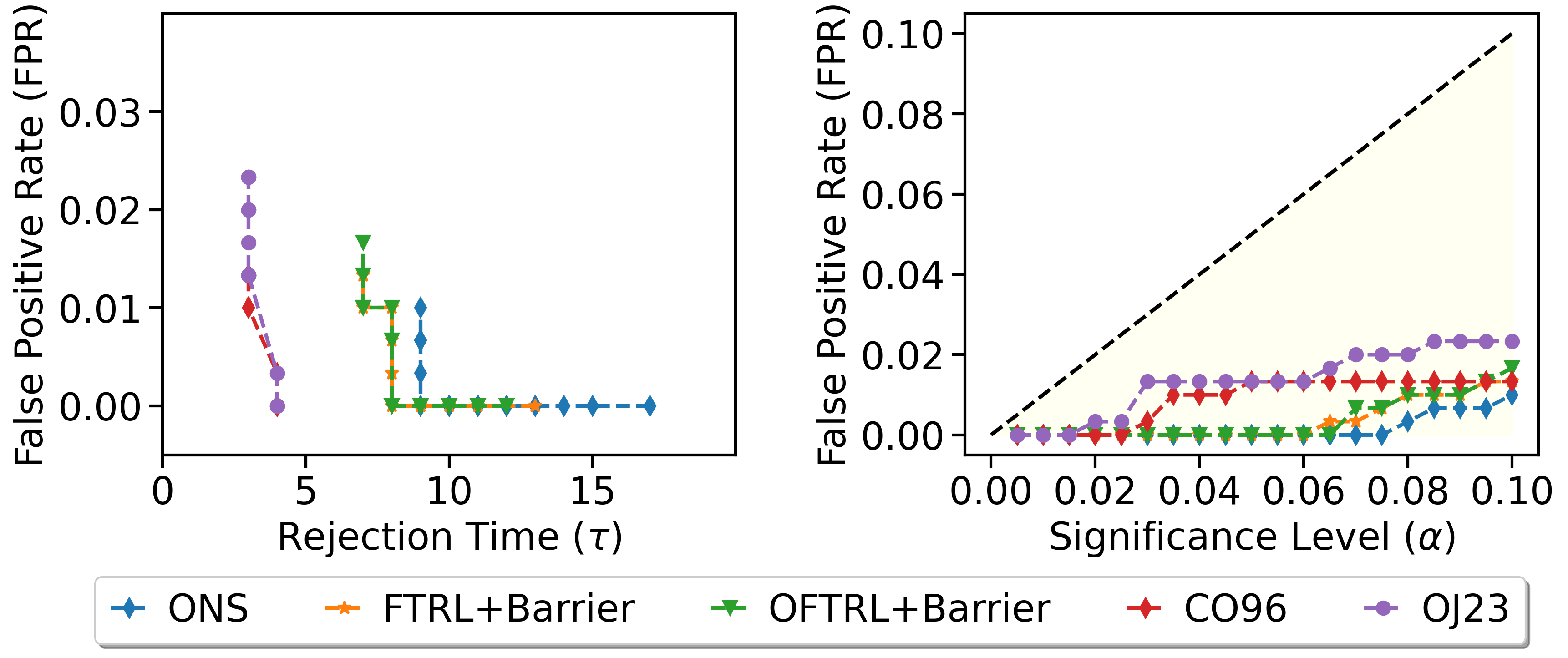}
\caption{One-sided testing 
$H_0: \mu_x < 0.1$, where $x_t \sim \mathrm{Bernoulli}(0.09)$ under $H_0$ and $x_t \sim \mathrm{Bernoulli}(0.95)$ under $H_1$.
}
\label{exp:one-sided-easy}
\end{figure}

\begin{figure}[t]
\centering
\includegraphics[width=0.8\textwidth]{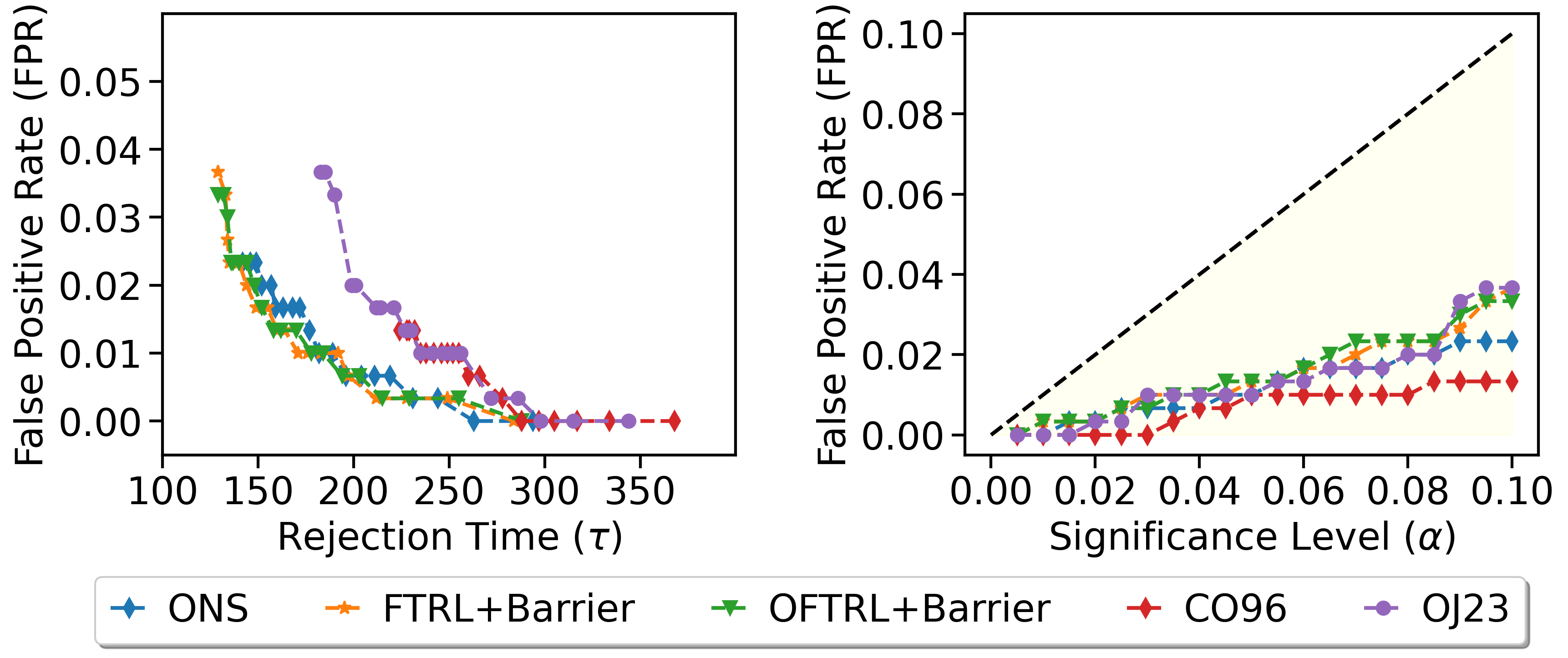}
\caption{
One-sided testing 
$H_0: \mu_x < 0.3$, where $x_t \sim \mathrm{Bernoulli}(0.29)$ under $H_0$ and $x_t \sim \mathrm{Bernoulli}(0.4)$ under $H_1$.
}
\label{exp:one-sided-hard}
\end{figure}

\begin{table}[htbp]
\centering
\caption{Runtime per iteration runtime in milliseconds for one-sided testing.}
\label{tab:avg-iteration-runtime2}
\begin{tabular}{lcc}
\toprule
\textbf{Method} & \textbf{Average (ms)} & \textbf{Std. Dev. (ms)} \\
\midrule
FTRL+Barrier   & 0.001856 & 0.000502 \\
OFTRL+Barrier  & 0.002424 & 0.000613 \\
ONS            & 0.001153 & 0.000433 \\
CO96           & 0.194680 & 0.048272 \\
OJ23           & 1.117311 & 0.730003 \\
\bottomrule
\end{tabular}
\end{table}

Figure~\ref{exp:one-sided-testing} plots the empirical distribution of rejection times for each method with significance level parameter $\alpha = 0.01$.  
We observe that FTRL+Barrier and OFTRL+Barrier generally lead to faster rejection of the null hypothesis, compared to the baselines.

\begin{figure}[t]
\centering
\includegraphics[width=0.8\textwidth]{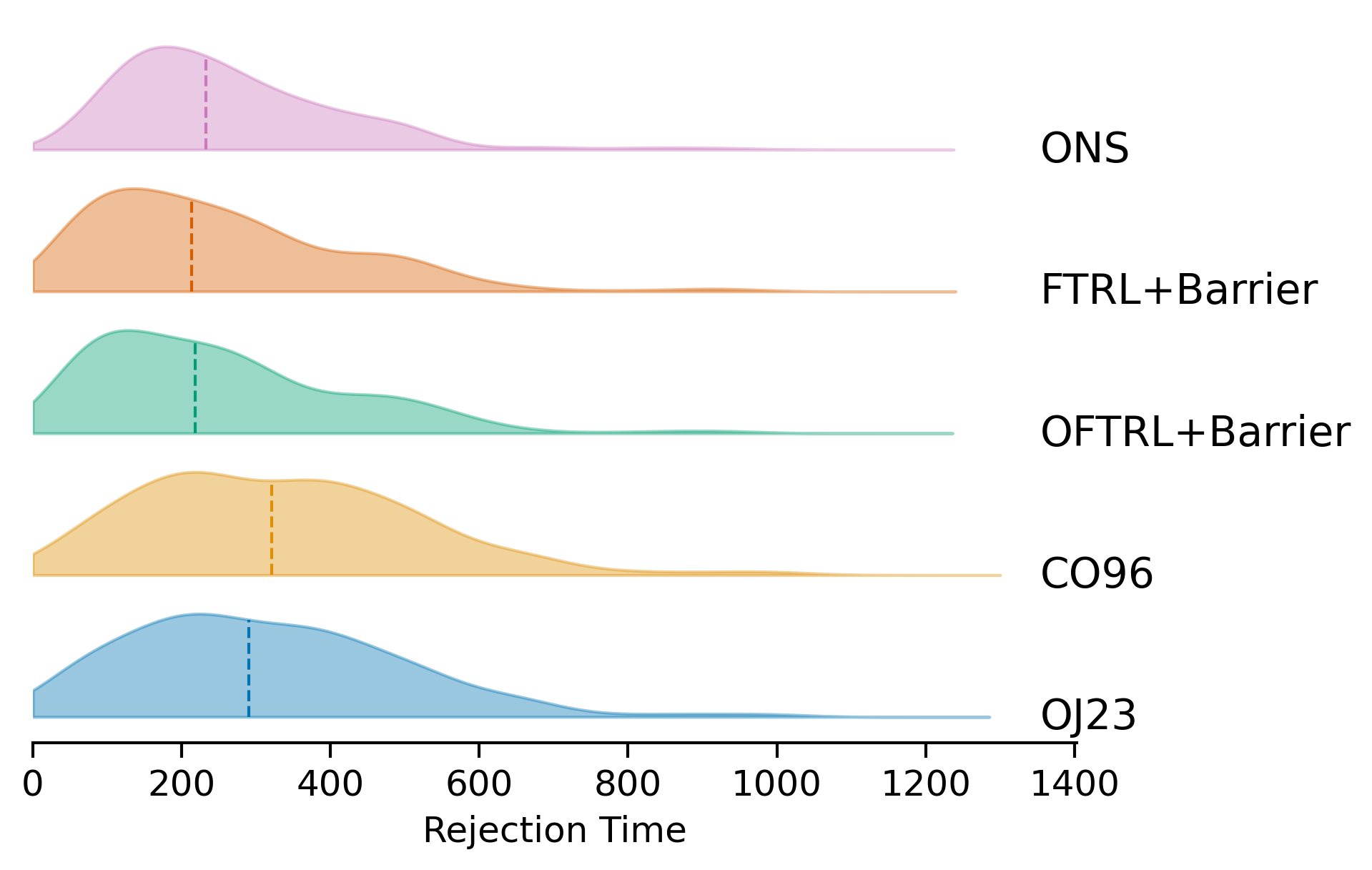}
\caption{
One-sided testing.
Testing against $H_0: \mu_x < 0.3$, where $x_t \sim \mathrm{Bernoulli}(0.4)$.
}
\label{exp:one-sided-testing}
\end{figure}

\section{Conclusion}

Motivated by interior-point methods in optimization, we propose two new strategies,
FTRL+Barrier and OFTRL+Barrier, for sequential hypothesis testing by betting. These strategies enable updates over a larger decision space compared to ONS while avoiding loss explosion and maintaining a similarly lightweight computational cost due to their closed-form updates.  
Our simulation results demonstrate that the proposed methods consistently outperform ONS.  
Furthermore, our empirical findings indicate that these methods are competitive with other more sophisticated betting-based testing approaches in balancing performance under both $H_0$ and $H_1$, while incurring significantly lower per-round runtime costs.  
Future work includes developing improved strategies for obtaining a better estimate of the next gradient for OFTRL+Barrier, while establishing a sublinear regret guarantee.

\section*{Acknowledgments}

The authors appreciate the support from NSF CCF-2403392, as well as the Google Gemma Academic Program and Google Cloud Credits.

\bibliographystyle{alpha}
\bibliography{main}

\newcommand{\etalchar}[1]{$^{#1}$}
\begin{thebibliography}{VEVdHKK20}

\bibitem[AHKS06]{agarwal2006algorithms}
Amit Agarwal, Elad Hazan, Satyen Kale, and Robert~E Schapire.
\newblock Algorithms for portfolio management based on the newton method.
\newblock In {\em Proceedings of the 23rd international conference on Machine
  learning}, pages 9--16, 2006.

\bibitem[AHR12]{abernethy2012interior}
Jacob~D Abernethy, Elad Hazan, and Alexander Rakhlin.
\newblock Interior-point methods for full-information and bandit online
  learning.
\newblock {\em IEEE Transactions on Information Theory}, 58(7):4164--4175,
  2012.

\bibitem[AR25]{agrawal2025stopping}
Shubhada Agrawal and Aaditya Ramdas.
\newblock On stopping times of power-one sequential tests: Tight lower and
  upper bounds.
\newblock {\em arXiv preprint arXiv:2504.19952}, 2025.

\bibitem[Bac10]{bach2010self}
Francis Bach.
\newblock Self-concordant analysis for logistic regression.
\newblock 2010.

\bibitem[BSR24]{bar2024protected}
Yarin Bar, Shalev Shaer, and Yaniv Romano.
\newblock Protected test-time adaptation via online entropy matching: A betting
  approach.
\newblock {\em NeurIPS}, 2024.

\bibitem[BV04]{boyd2004convex}
Stephen~P Boyd and Lieven Vandenberghe.
\newblock {\em Convex optimization}.
\newblock Cambridge university press, 2004.

\bibitem[CCGWR23]{Chugg2023}
Ben Chugg, Santiago Cortes-Gomez, Bryan Wilder, and Aaditya Ramdas.
\newblock Auditing fairness by betting.
\newblock {\em Advances in Neural Information Processing Systems},
  36:6070--6091, 2023.

\bibitem[CGK24]{cho2024peeking}
Brian Cho, Kyra Gan, and Nathan Kallus.
\newblock Peeking with peak: Sequential, nonparametric composite hypothesis
  tests for means of multiple data streams.
\newblock {\em ICML}, 2024.

\bibitem[CO96]{cover1996universal}
Thomas~M Cover and Erik Ordentlich.
\newblock Universal portfolios with side information.
\newblock {\em IEEE Transactions on Information Theory}, 42(2), 1996.

\bibitem[CO18]{Cutkosky2018}
Ashok Cutkosky and Francesco Orabona.
\newblock Black-box reductions for parameter-free online learning in banach
  spaces.
\newblock In {\em Conference On Learning Theory}, pages 1493--1529. PMLR, 2018.

\bibitem[Cov74]{cover1974universal}
Thomas~M Cover.
\newblock Universal gambling schemes and the complexity measures of kolmogorov
  and chaitin.
\newblock {\em Technical Report, no. 12}, 1974.

\bibitem[Cov91]{cover1991universal}
Thomas~M Cover.
\newblock Universal portfolios.
\newblock {\em Mathematical finance}, 1(1):1--29, 1991.

\bibitem[CW25]{CW2025}
Can Chen and Jun-Kun Wang.
\newblock Online detection of {LLM}-generated texts via sequential hypothesis
  testing by betting.
\newblock {\em ICML}, 2025.

\bibitem[CYL{\etalchar{+}}12]{chiang2012online}
Chao-Kai Chiang, Tianbao Yang, Chia-Jung Lee, Mehrdad Mahdavi, Chi-Jen Lu, Rong
  Jin, and Shenghuo Zhu.
\newblock Online optimization with gradual variations.
\newblock In {\em Conference on Learning Theory}, pages 6--1. JMLR Workshop and
  Conference Proceedings, 2012.

\bibitem[CZT{\etalchar{+}}24]{chen2024optimistic}
Sijia Chen, Yu-Jie Zhang, Wei-Wei Tu, Peng Zhao, and Lijun Zhang.
\newblock Optimistic online mirror descent for bridging stochastic and
  adversarial online convex optimization.
\newblock {\em Journal of Machine Learning Research}, 25(178):1--62, 2024.

\bibitem[DGRR25]{dai2025individual}
Jessica Dai, Paula Gradu, Inioluwa~Deborah Raji, and Benjamin Recht.
\newblock From individual experience to collective evidence: A reporting-based
  framework for identifying systemic harms.
\newblock {\em ICML}, 2025.

\bibitem[FR24]{fischer2024multiple}
Lasse Fischer and Aaditya Ramdas.
\newblock Multiple testing with anytime-valid monte-carlo p-values.
\newblock {\em arXiv preprint arXiv:2404.15586}, 2024.

\bibitem[GdHK24]{grunwald2024authors}
Peter Grünwald, Rianne de~Heide, and Wouter Koolen.
\newblock Safe testing.
\newblock {\em Journal of the Royal Statistical Society Series B: Statistical
  Methodology}, 86(5):1091--1128, 2024.

\bibitem[GHL24]{grunwald2024anytime}
Peter Gr{\"u}nwald, Alexander Henzi, and Tyron Lardy.
\newblock Anytime-valid tests of conditional independence under model-x.
\newblock {\em Journal of the American Statistical Association},
  119(546):1554--1565, 2024.

\bibitem[H{\etalchar{+}}16]{Hazan2016OnlineConvexOptimization}
Elad Hazan et~al.
\newblock Introduction to online convex optimization.
\newblock {\em Foundations and Trends{\textregistered} in Optimization},
  2(3-4):157--325, 2016.

\bibitem[HAK07]{hazan2007logarithmic}
Elad Hazan, Amit Agarwal, and Satyen Kale.
\newblock Logarithmic regret algorithms for online convex optimization.
\newblock {\em Machine Learning}, 69(2):169--192, 2007.

\bibitem[HK10]{hazan2010extracting}
Elad Hazan and Satyen Kale.
\newblock Extracting certainty from uncertainty: Regret bounded by variation in
  costs.
\newblock {\em Machine learning}, 80:165--188, 2010.

\bibitem[HK15]{hazan2015online}
Elad Hazan and Satyen Kale.
\newblock An online portfolio selection algorithm with regret logarithmic in
  price variation.
\newblock {\em Mathematical Finance}, 25(2):288--310, 2015.

\bibitem[IHS{\etalchar{+}}18]{ito2018regret}
Shinji Ito, Daisuke Hatano, Hanna Sumita, Akihiro Yabe, Takuro Fukunaga,
  Naonori Kakimura, and Ken-ichi Kawarabayashi.
\newblock Regret bounds for online portfolio selection with a cardinality
  constraint.
\newblock {\em Advances in Neural Information Processing Systems}, 31, 2018.

\bibitem[JGS17]{joulani2017modular}
Pooria Joulani, Andr{\'a}s Gy{\"o}rgy, and Csaba Szepesv{\'a}ri.
\newblock A modular analysis of adaptive (non-) convex optimization: Optimism,
  composite objectives, and variational bounds.
\newblock In {\em International Conference on Algorithmic Learning Theory},
  pages 681--720, 2017.

\bibitem[JJKO23]{jang2023tighter}
Kyoungseok Jang, Kwang-Sung Jun, Ilja Kuzborskij, and Francesco Orabona.
\newblock Tighter pac-bayes bounds through coin-betting.
\newblock In {\em The Thirty Sixth Annual Conference on Learning Theory}, pages
  2240--2264. PMLR, 2023.

\bibitem[JOG22]{jezequel2022efficient}
R{\'e}mi J{\'e}z{\'e}quel, Dmitrii~M Ostrovskii, and Pierre Gaillard.
\newblock Efficient and near-optimal online portfolio selection.
\newblock {\em arXiv preprint arXiv:2209.13932}, 2022.

\bibitem[Kel56]{kelly1956new}
John~L Kelly.
\newblock A new interpretation of information rate.
\newblock {\em the bell system technical journal}, 35(4):917--926, 1956.

\bibitem[LWZ18]{luo2018efficient}
Haipeng Luo, Chen-Yu Wei, and Kai Zheng.
\newblock Efficient online portfolio with logarithmic regret.
\newblock {\em Advances in neural information processing systems}, 31, 2018.

\bibitem[McM17]{mcmahan2017survey}
H~Brendan McMahan.
\newblock A survey of algorithms and analysis for adaptive online learning.
\newblock {\em Journal of Machine Learning Research}, 18(90):1--50, 2017.

\bibitem[MR22]{mhammedi2022damped}
Zakaria Mhammedi and Alexander Rakhlin.
\newblock Damped online newton step for portfolio selection.
\newblock In {\em Conference on learning theory}, pages 5561--5595. PMLR, 2022.

\bibitem[Nem04]{nemirovski2004interior}
Arkadi Nemirovski.
\newblock Interior point polynomial time methods in convex programming.
\newblock {\em Lecture notes}, 42(16):3215--3224, 2004.

\bibitem[Nes13]{nesterov2013introductory}
Yurii Nesterov.
\newblock {\em Introductory lectures on convex optimization: A basic course},
  volume~87.
\newblock Springer Science \& Business Media, 2013.

\bibitem[NN94]{nesterov1994interior}
Yurii Nesterov and Arkadii Nemirovski.
\newblock {\em Interior-point polynomial algorithms in convex programming}.
\newblock SIAM, 1994.

\bibitem[NT08]{nemirovski2008interior}
Arkadi~S Nemirovski and Michael~J Todd.
\newblock Interior-point methods for optimization.
\newblock {\em Acta Numerica}, 17:191--234, 2008.

\bibitem[OJ23]{orabona2023tight}
Francesco Orabona and Kwang-Sung Jun.
\newblock Tight concentrations and confidence sequences from the regret of
  universal portfolio.
\newblock {\em IEEE Transactions on Information Theory}, 2023.

\bibitem[OP16]{orabona2016coin}
Francesco Orabona and D{\'a}vid P{\'a}l.
\newblock Coin betting and parameter-free online learning.
\newblock {\em Advances in Neural Information Processing Systems}, 29, 2016.

\bibitem[Ora19]{orabona2019modern}
Francesco Orabona.
\newblock A modern introduction to online learning.
\newblock {\em arXiv preprint arXiv:1912.13213}, 2019.

\bibitem[PBKR23]{podkopaev2024sequential}
Aleksandr Podkopaev, Patrick Bl{\"o}baum, Shiva Kasiviswanathan, and Aaditya
  Ramdas.
\newblock Sequential kernelized independence testing.
\newblock In {\em International Conference on Machine Learning}, pages
  27957--27993. PMLR, 2023.

\bibitem[PFRS24]{pandeva2024deep}
Teodora Pandeva, Patrick Forr{\'e}, Aaditya Ramdas, and Shubhanshu Shekhar.
\newblock Deep anytime-valid hypothesis testing.
\newblock In {\em International Conference on Artificial Intelligence and
  Statistics}, pages 622--630. PMLR, 2024.

\bibitem[PR23]{podkopaev2023sequential}
Aleksandr Podkopaev and Aaditya Ramdas.
\newblock Sequential predictive two-sample and independence testing.
\newblock {\em Advances in neural information processing systems},
  36:53275--53307, 2023.

\bibitem[PXL24]{PXL24}
Aleksandr Podkopaev, Dong Xu, and Kuang-Chih Lee.
\newblock Adaptive conformal inference by betting.
\newblock {\em ICML}, 2024.

\bibitem[RGVS23]{ramdas2023game}
Aaditya Ramdas, Peter Gr{\"u}nwald, Vladimir Vovk, and Glenn Shafer.
\newblock Game-theoretic statistics and safe anytime-valid inference.
\newblock {\em Statistical Science}, 38(4):576--601, 2023.

\bibitem[RM23]{RamdasManole2023}
Aaditya Ramdas and Tudor Manole.
\newblock Randomized and exchangeable improvements of markov's, chebyshev's and
  chernoff's inequalities.
\newblock {\em arXiv preprint arXiv:2304.02611}, 2023.

\bibitem[RRLK20]{ramdas2020admissible}
Aaditya Ramdas, Johannes Ruf, Martin Larsson, and Wouter Koolen.
\newblock Admissible anytime-valid sequential inference must rely on
  nonnegative martingales.
\newblock {\em arXiv preprint arXiv:2009.03167}, 2020.

\bibitem[RS74]{robbins1974expected}
Herbert Robbins and David Siegmund.
\newblock The expected sample size of some tests of power one.
\newblock {\em The Annals of Statistics}, 2(3):415--436, 1974.

\bibitem[RS13]{rakhlin2013online}
Alexander Rakhlin and Karthik Sridharan.
\newblock Online learning with predictable sequences.
\newblock In {\em Conference on Learning Theory}, pages 993--1019. PMLR, 2013.

\bibitem[RW24]{ramdas2024hypothesis}
Aaditya Ramdas and Ruodu Wang.
\newblock Hypothesis testing with e-values.
\newblock {\em arXiv preprint arXiv:2410.23614}, 2024.

\bibitem[SALS15]{syrgkanis2015fast}
Vasilis Syrgkanis, Alekh Agarwal, Haipeng Luo, and Robert~E Schapire.
\newblock Fast convergence of regularized learning in games.
\newblock {\em Advances in Neural Information Processing Systems}, 28, 2015.

\bibitem[Sha21]{shafer2021testing}
Glenn Shafer.
\newblock Testing by betting: A strategy for statistical and scientific
  communication.
\newblock {\em Journal of the Royal Statistical Society Series A: Statistics in
  Society}, 184(2):407--431, 2021.

\bibitem[SMR23]{shaer2023model}
Shalev Shaer, Gal Maman, and Yaniv Romano.
\newblock Model-{X} sequential testing for conditional independence via testing
  by betting.
\newblock In {\em International Conference on Artificial Intelligence and
  Statistics}, pages 2054--2086. PMLR, 2023.

\bibitem[SR23]{shekhar2023nonparametric}
Shubhanshu Shekhar and Aaditya Ramdas.
\newblock Nonparametric two-sample testing by betting.
\newblock {\em IEEE Transactions on Information Theory}, 70(2):1178--1203,
  2023.

\bibitem[SSS06]{shalev2006online}
Shai Shalev-Shwartz and Yoram Singer.
\newblock Online learning meets optimization in the dual.
\newblock In {\em International Conference on Computational Learning Theory},
  pages 423--437. Springer, 2006.

\bibitem[SV19]{shafer2019game}
Glenn Shafer and Vladimir Vovk.
\newblock {\em Game-theoretic foundations for probability and finance}, volume
  455.
\newblock John Wiley \& Sons, 2019.

\bibitem[TLL23]{tsai2023data}
Chung-En Tsai, Ying-Ting Lin, and Yen-Huan Li.
\newblock Data-dependent bounds for online portfolio selection without
  lipschitzness and smoothness.
\newblock {\em Advances in Neural Information Processing Systems},
  36:62764--62791, 2023.

\bibitem[TS24]{teneggi2024bet}
Jacopo Teneggi and Jeremias Sulam.
\newblock I bet you did not mean that: Testing semantic importance via betting.
\newblock {\em NeurIPS}, 2024.

\bibitem[VEVdHKK20]{van2020open}
Tim Van~Erven, Dirk Van~der Hoeven, Wojciech Kot{\l}owski, and Wouter~M Koolen.
\newblock Open problem: Fast and optimal online portfolio selection.
\newblock In {\em Conference on learning theory}, pages 3864--3869. PMLR, 2020.

\bibitem[Vil39]{Ville1939}
Jean Ville.
\newblock {\em Etude critique de la notion de collectif}.
\newblock Gauthier-Villars Paris, 1939.

\bibitem[VW21]{vovk2021values}
Vladimir Vovk and Ruodu Wang.
\newblock E-values: Calibration, combination and applications.
\newblock {\em The Annals of Statistics}, 49(3):1736--1754, 2021.

\bibitem[WA18]{wang2018acceleration}
Jun-Kun Wang and Jacob~D Abernethy.
\newblock Acceleration through optimistic no-regret dynamics.
\newblock {\em Advances in Neural Information Processing Systems}, 31, 2018.

\bibitem[WAL24]{wang2024no}
Jun-Kun Wang, Jacob Abernethy, and Kfir~Y Levy.
\newblock No-regret dynamics in the {Fenchel} game: A unified framework for
  algorithmic convex optimization.
\newblock {\em Mathematical Programming}, 205(1):203--268, 2024.

\bibitem[WRB20]{wasserman2020universal}
Larry Wasserman, Aaditya Ramdas, and Sivaraman Balakrishnan.
\newblock Universal inference.
\newblock {\em Proceedings of the National Academy of Sciences},
  117(29):16880--16890, 2020.

\bibitem[Wri97]{wright1997primal}
Stephen~J Wright.
\newblock {\em Primal-dual interior-point methods}.
\newblock SIAM, 1997.

\bibitem[WSR24]{waudby2024estimating}
Ian Waudby-Smith and Aaditya Ramdas.
\newblock Estimating means of bounded random variables by betting.
\newblock {\em Journal of the Royal Statistical Society Series B: Statistical
  Methodology}, 86(1):1--27, 2024.

\bibitem[WSSJ25]{waudby2025universal}
Ian Waudby-Smith, Ricardo Sandoval, and Michael~I Jordan.
\newblock Universal log-optimality for general classes of e-processes and
  sequential hypothesis tests.
\newblock {\em arXiv preprint arXiv:2504.02818}, 2025.

\bibitem[ZAK22]{zimmert2022pushing}
Julian Zimmert, Naman Agarwal, and Satyen Kale.
\newblock Pushing the efficiency-regret pareto frontier for online learning of
  portfolios and quantum states.
\newblock In {\em Conference on Learning Theory}, pages 182--226. PMLR, 2022.

\end{thebibliography}

\appendix

\section{Proof of Lemma~\ref{lem:close-form2}} \label{app:lem:close-form2}

\begin{proof}
In the following, we denote $a:=\eta\sum_{s=1}^t \nabla \ell_s(\theta_s)$. Then,
the Lagrangian of the objective function for FTRL+Barrier (line 6 in Algorithm~\ref{alg:FTRL})
over the constraint set $\K= [0,1]$ is $\mathcal{L}(\theta, \lambda_1, \lambda_2) = a\theta - \ln(\theta) - \ln(1-\theta) - \lambda_1(\theta) + \lambda_2(\theta-1)$,
where $\lambda_1, \lambda_2 \geq 0$ 
are Lagrange multipliers.

The complementary slackness of the KKT optimality condition is
 \begin{equation} \label{slackness-2}
\lambda_1(\theta) = 0 \qquad \text{ and  } \qquad \lambda_2(\theta - 1) = 0,
\end{equation}
while the stationarity of the KKT condition is 
\begin{equation}
\frac{\partial \mathcal{L}}{\partial \theta} = a - \frac{1}{\theta} + \frac{1}{1 - \theta} - \lambda_1 +\lambda_2 = 0. \label{stationarity-2}
\end{equation}
Observe that the stationarity condition \eqref{stationarity-2} cannot be satisfied at $\theta = 0$ and $1$,
hence the solution $\theta \neq 0, 1$, which in turn implies that $\lambda_1=\lambda_2 = 0$ by 
the complementary slackness \eqref{slackness-2}.
Then, solving \eqref{stationarity-2} becomes solving a quadratic equation, i.e.,
$a\theta^2 - (a+2)\theta + 1 = 0$,
which has two roots. 
However, by the primal feasibility condition $\theta\in[0,1]$, only $\theta = \frac{a+2 - \sqrt{4+a^2}}{2a}$ is the valid solution. 
When $a=0$, one can easily conclude that $\theta=1/2$ from \eqref{stationarity-2}.
We hence have completed the proof.
\end{proof}

\section{Proof of Lemma~\ref{lem:barrier}} \label{app:lem:barrier} 

The proof requires the following property regarding self-concordant functions.
\begin{lemma}[Adapted from Theorem 5.1.1 in \cite{nesterov2013introductory}] \label{lem:intersection}
Suppose functions $f(\cdot): \K_f \to \reals $ and $g(\cdot): \K_g \to \reals$ are self-concordant.
Then, $F(\cdot)= f(\cdot) + g(\cdot)$ is self-concordant over the set $\K_f \cap \K_g$.
\end{lemma}

\begin{proof}(of Lemma~\ref{lem:barrier})
Consider $r_{-1}(x):= - \ln (1+x)$.
Its second derivative is $(1+x)^{-2}$, and its third derivative is $-2 / (1+x)^{-3}$. Therefore, $r_{-1}(x)$ is a self-concordant function. Similarly, $r_{1}(x):= - \ln (1-x)$ has its second derivative equal to $(1-x)^{-2}$ and third derivative equal to $2 / (1-x)^3$. Hence, $r_{1}(x)$ is a self-concordant function. Furthermore, $r_0(x):= - \ln (x)$ can be easily checked to be a self-concordant function.

To complete the proof, we apply Lemma~\ref{lem:intersection}: once for $R(x) = -\ln(1 + x) - \ln(1 - x) = r_{-1}(x) + r_1(x)$, which shows that the barrier function is self-concordant on $\K := [-1,1]$; and once for $R(x) = -\ln(x) - \ln(1 - x) = r_{0}(x) + r_1(x)$, which shows that the barrier function is self-concordant on $\K := [0,1]$.
\end{proof}

\section{Proof of Lemma~\ref{lem:FTRL_const} } \label{supp:A}

We first replicate Lemma~\ref{lem:FTRL_const} here.

\begin{mdframed}[backgroundcolor=black!10,rightline=false,leftline=false,topline=false,bottomline=false]
\noindent
\textbf{Lemma~\ref{lem:FTRL_const}}:
\textit{
Denote $G_t:= \sum_{s=1}^t \nabla \ell(\theta_s) $. 
Suppose that there exists a time point $t_0$ such that for all $t \geq t_0$, we have
\begin{equation} \label{growth}
\textbf{(Linear growth of cumulative gradients)} \qquad 
\left | G_t   \right| \geq ct,
\end{equation}
for some $c>0$. 
Then, FTRL+Barrier (Algorithm~\ref{alg:FTRL}) satisfies  
$$\mathrm{Regret}_T(\theta_*) 
\leq \frac{t_0}{8\eta} + \frac{2}{c'^2\eta} \left( \frac{1}{t_0-1} - \frac{1}{T-1} \right )
+ \frac{  R(\theta_*)  }{\eta},$$
where $\eta\leq \frac{1}{4}$ is the parameter, $c' > 0$ is a constant, and $\theta_* \in K$ is any comparator.}
\end{mdframed}

\begin{proof}
For the difference-in-means testing,
$ R(\theta) = -\ln(1 - \theta) - \ln(1 + \theta)$ 
and its Hessian is
\begin{equation}
    \nabla^2 R(\theta)=\frac{1}{(1-\theta)^2}+\frac{1}{(1+\theta)^2} = 
    \frac{2+2\theta^2}{(1-\theta^2)^2}.
    \label{eq:hessian_r}
\end{equation}
In the following, we denote  
$x:= \eta G_{t-1}$ for notation brevity. 
Recall the closed-form expression of $\theta_t$ (Lemma~\ref{lem:close-form}), i.e., $\theta_t = \frac{ 1 - \sqrt{ 1 + x^2  }  }{x} $.
Then, the denominator of the Hessian can be upper-bounded as:
\begin{align}
( 1 - \theta_t^2 )^2 & = \left( 1 - \left( \frac{1 - \sqrt{1+x^2} }{x}   \right)^2  \right)^2 \notag
 = \left( \frac{ 2 \sqrt{1+x^2} -2  }{x^2} \right)^2   \notag
\\ & \leq 
\left( \frac{ 2 + 2 |x| -2  }{x^2} \right)^2   \notag
\\ & \leq \frac{1}{x^2}, \label{denominator}
\end{align}
where the second-to-last inequality above uses $\sqrt{a+b} \leq \sqrt{a} + \sqrt{b}$ for any $a,b \geq 0$.
Combining \eqref{eq:hessian_r} and \eqref{denominator},
we further have
\begin{align}
    \nabla^2 R(\theta_t)=\frac{2+2\theta_t^2}{(1-\theta_t^2)^2}& \overset{(i)}{\geq} \frac{2+2\theta_t^2}{  \frac{1}{(\eta G_{t-1})^2} } \geq 2 \eta_t^2 G_{t-1}^2
    \\ & \overset{(ii)}{\geq} 2 \eta^2 c^2 (t-1)^2 \label{eq:1d_hessian},
\end{align}
where (i) uses \eqref{denominator} and (ii) uses 
the condition of a linear growth of cumulative gradients.

Using the above results, 
when the linear growth condition holds,
we have
\begin{align}
 \|\nabla \ell_t\|_{\theta_t}^{*2} = (\nabla \ell_t(\theta_t))^2 ( \nabla^2 R(\theta_t) )^{-1} \leq \left(\frac{g_t}{1-g_t\theta_t}\right)^2\cdot \frac{1}{ 2 c^2\eta^2 (t-1)^2} & \leq \frac{1}{2(1-\theta_t)^2 c^2\eta^2(t-1)^2} \notag  \\
\\ & \leq \frac{1}{2 (c')^2 \eta^2 (t-1)^2},
 \label{eq:norm_bound}
\end{align}
where the last one is based on the update of $\theta_{t}$, there must exist some constant $c'$ such that $(1-\theta_t)^2 c^2 \geq c'^2$ for all $t$.
Moreover, since $ \eta \|\nabla\ell_t\|_{\theta_t}^*= \eta\sqrt{(\nabla\ell_t)^2 (\nabla^2 R(\theta))^{-1}}\leq \frac{1}{4} $ by Lemma~\ref{lem:gradient_norm}, we always have $ {(\nabla\ell_t)^2 (\nabla^2 R(\theta_t))^{-1}}\leq \frac{1}{16\eta^2} $.  

Therefore, 
\begin{align} 
\mathrm{Regret}_T(\mathrm{FTRLBarrier})
&\leq 2 \eta \sum_{t=1}^{t_0}
(\nabla \ell_t(\theta_t))^2 ( \nabla^2 R(\theta_t) )^{-1} + 2\eta\sum_{t=t_0+1}^{T}
(\nabla \ell_t(\theta_t))^2 ( \nabla^2 R(\theta_t) )^{-1}+
\frac{  R(\theta_*)  }{\eta}\notag,
\\ &
\leq 2 \eta \sum_{t=1}^{t_0}
(\nabla \ell_t(\theta_t))^2 ( \nabla^2 R(\theta_t) )^{-1} + 
2\eta\sum_{t=t_0+1}^{T} \frac{1}{2 c'^2 \eta^2 (t-1)^2}
+
\frac{  R(\theta_*)  }{\eta}\notag,
\\ &
\leq 2 \eta \sum_{t=1}^{t_0}
(\nabla \ell_t(\theta_t))^2 ( \nabla^2 R(\theta_t) )^{-1} + 
\int_{t_0-1}^{T-1} \frac{1}{ c'^2 \eta t^2}
+
\frac{  R(\theta_*)  }{\eta}\notag
\\ &
\leq 2 t_0 \frac{1}{16\eta} + \frac{1}{c'^2\eta} \left( \frac{1}{t_0-1} - \frac{1}{T-1} \right )
+ \frac{  R(\theta_*)  }{\eta}\notag\\
&=\frac{t_0}{8\eta} + \frac{1}{c'^2\eta} \left( \frac{1}{t_0-1} - \frac{1}{T-1} \right )
+ \frac{  R(\theta_*)  }{\eta}\notag\\
&\leq \frac{1}{\eta}\left(\frac{t_0}{8} + \frac{1}{c'^2t_0} 
+  R(\theta_*)\right)\notag.
\end{align}
If we choose $\eta=\frac{1}{4}$, we have a constant regret bound $\frac{t_0}{2} + \frac{16}{c'^2} \frac{1}{t_0}
+ 4R(\theta_*)$. 

Now let us switch to the one-sided testing,
where we have $R(\theta):= -\ln (\theta) - \ln (1-\theta)$ and $K:=[0,1]$.
For the difference-in-means testing,
$ R(\theta) = -\ln(1 - \theta) - \ln(1 + \theta)$ 
and its Hessian is
\begin{equation}
    \nabla^2 R(\theta)=\frac{1}{\theta^2}+\frac{1}{(1-\theta)^2} = \frac{1 - 2 \theta + 2 \theta^2}{ (\theta - \theta^2)^2 }. \label{eq:hessian_r2}
\end{equation}
Recall the closed-form expression of $\theta_t$ (Lemma~\ref{lem:close-form2}), i.e., 
\begin{equation}
\theta_t = \frac{ 2 + \eta G_{t-1} - \sqrt{ 4 + (\eta G_{t-1})^2  }  }{2 \eta G_{t-1} }
= \frac{1}{2} + \frac{1}{\eta G_{t-1}}
- \frac{\sqrt{ 4 + (\eta G_{t-1})^2  }  }{2 \eta G_{t-1} }.
\end{equation}
Then, with a computer-aided analysis, the Hessian has a simplified expression:
\begin{align}
 \nabla^2 R(\theta_t)= (\eta G_{t-1})^2 + 2 \sqrt{ (\eta G_{t-1})^2  +4  } + 4.
\end{align}
Therefore, when the linear condition holds, we have
\begin{equation}
 \nabla^2 R(\theta_t)
 \geq \eta^2 c^2 (t-1)^2 + 2 \eta c (t-1)
 \geq  \eta^2 c^2 (t-1)^2.
 \end{equation}

The remaining analysis of the regret bound essentially follows that for the difference-in-means testing above.

We have now completed the proof.

\end{proof}

\end{document}